\documentclass[11pt]{article}
\usepackage[preprint]{acl}
\usepackage{times,latexsym}
\usepackage[T1]{fontenc}
\usepackage[utf8]{inputenc}
\usepackage{microtype,inconsolata}
\usepackage{xurl}

\usepackage{amsmath,amsfonts,bm}

\def\eqref#1{equation~\ref{#1}}

\def\1{\bm{1}}

\DeclareMathAlphabet{\mathsfit}{\encodingdefault}{\sfdefault}{m}{sl}
\SetMathAlphabet{\mathsfit}{bold}{\encodingdefault}{\sfdefault}{bx}{n}

\usepackage{hyperref,url,amsmath,amssymb,amsthm}
\usepackage{xcolor,booktabs,array,graphicx}
\usepackage{colortbl}
\definecolor{ourblue}{rgb}{0.368,0.507,0.71}
\definecolor{tablegroup}{RGB}{255,247,204}
\definecolor{tableroot}{RGB}{225,237,245}
\hypersetup{colorlinks=true,citecolor=ourblue,linkcolor=ourblue,urlcolor=ourblue}
\usepackage{placeins,flafter}
\newtheorem{theorem}{Theorem}

\newtheorem{definition}{Definition}

\newcolumntype{L}[1]{>{\raggedright\arraybackslash}p{#1}}
\title{Understanding Target Probability Reuse in Speculative Decoding}
\author{
  Jungseob Lee\textsuperscript{1} \quad Giyoon Jang\textsuperscript{2} \quad Chailin Nam\textsuperscript{2} \\
  \bfseries Jiwoo Kum\textsuperscript{2} \quad Sihyeon Lee\textsuperscript{2} \quad Kyojun Choo\textsuperscript{2} \\
  \bfseries Jihun Song\textsuperscript{2} \quad Eunseo Song\textsuperscript{2} \quad Chanjun Park\textsuperscript{2}\thanks{Corresponding author.} \\
  \normalfont\textsuperscript{1}Korea University \quad \textsuperscript{2}Soongsil University \\
  {\small\texttt{omanma1928@korea.ac.kr} \quad \texttt{jyjang1130@gmail.com} \quad \texttt{skacofls26@naver.com}} \\
  {\small\texttt{rmawldn1204@naver.com} \quad \texttt{sihyeon2988@gmail.com} \quad \texttt{cnrywns1234@naver.com}} \\
  {\small\texttt{songjh041207@gmail.com} \quad \texttt{eunseosong4916@gmail.com} \quad \texttt{chanjun.park@ssu.ac.kr}}
}
\begin{document}
\raggedbottom
\setcounter{footnote}{1} 
\maketitle
\begin{abstract}
Speculative decoding uses a target model to verify a tree of draft tokens in parallel. We study whether the resulting next-token distributions can help select the next tree. For unrestricted target models, we show that guaranteed exact reuse requires a distribution computed at the next round's complete context. We characterize when retaining this distribution requires sacrificing accepted tokens and establish conditions under which this loss persists across decoding rounds. Across six targets, we find that reranking first-token branches with a separately computed distribution at the current context improves expected acceptance over draft ranking, especially with small tree budgets. Drawing the first token from this distribution before allocating the tree further improves expected acceptance over pruning a preselected tree, even at equal mean tree size. In our experiments, both selection rules also decode faster than draft ranking when the separate target pass overlaps standalone drafting. Our findings show that making target probabilities available before tree allocation can improve the efficiency of speculative decoding. Our code is available at \url{https://github.com/js-lee-AI/TargetRows}.
\end{abstract}
\begin{figure}[t]
\centering
\includegraphics[width=\columnwidth]{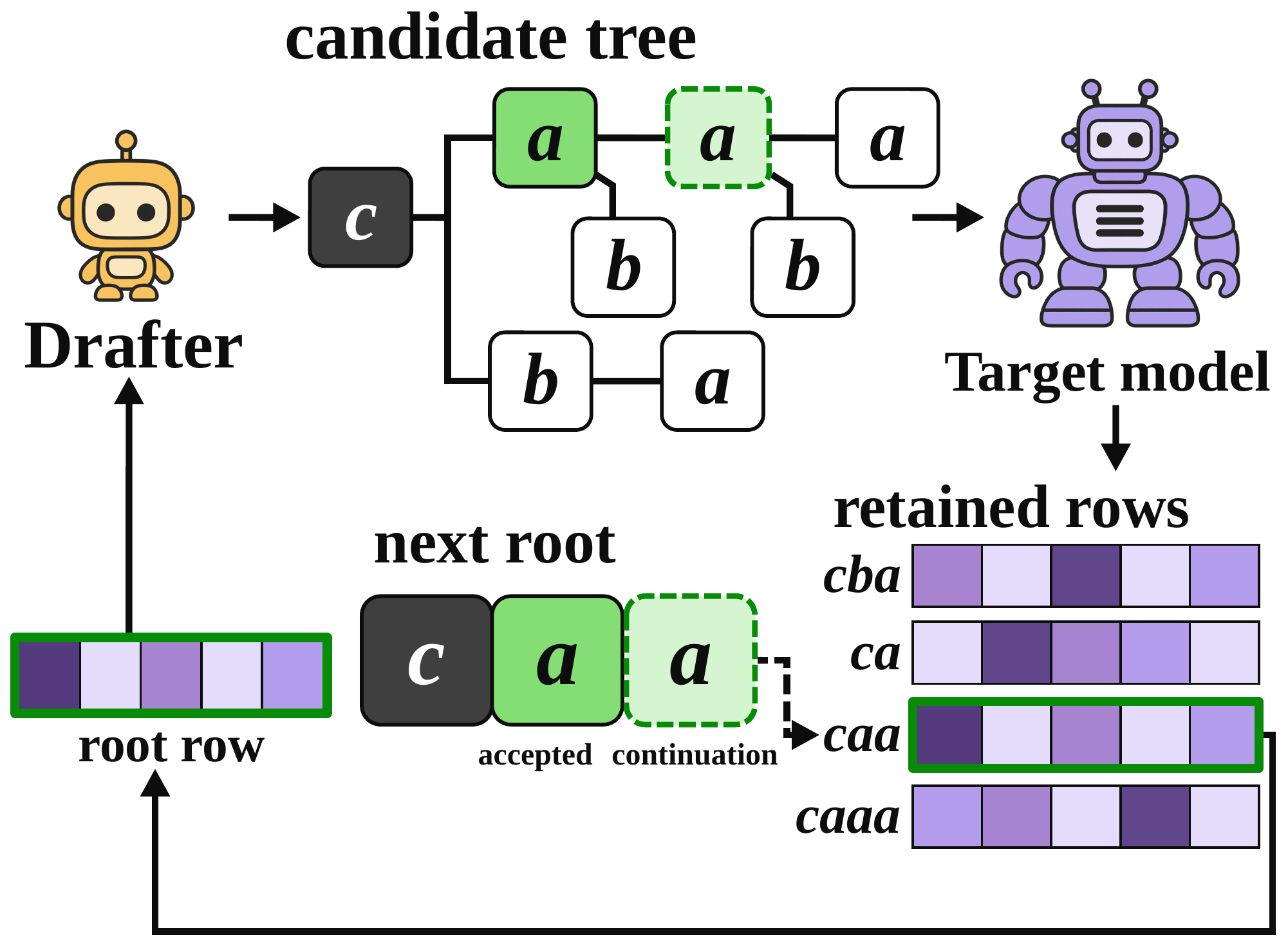}
\caption{Address matching across decoding rounds. The dashed link identifies the scored prefix $caa$ with the next root; its row is $p(\cdot\mid caa)$.}
\label{fig:roundmap}
\end{figure}

\section{Introduction}
\label{sec:intro}
Large language models typically generate text one token at a time \citep{lee2025clo,lee2026bytecross,lee2026acoer,lee2026answerconditioned,kang2026thinkfuse}. Speculative decoding accelerates this process by letting a cheap drafter propose tokens for the target model to verify in parallel \citep{leviathan2023fast,chen2023accelerating,lee2026redline}. Tree-based methods extend this approach by arranging several candidate continuations into branches and verifying them in a single target pass \citep{miao2024specinfer,cai2024medusa,li2024eagle2,li2025eagle3,lee2026cast}. To keep the tree within a fixed node budget, the drafter must choose which branches to retain. The target's probabilities could guide this choice, yet the tree is usually selected before verification reveals them. The previous pass offers another source of information because it has already computed next-token distributions at candidate prefixes. We call these distributions rows.

We ask when one of these rows can provide exact next-token probabilities for selecting the next tree without another target evaluation. We find that, for an unrestricted target, such reuse is guaranteed only when the row was computed at the next round's complete context, which we call its address. Two addresses may end in the same token even when they contain different sequences of preceding tokens. A cache indexed by the final token can still provide draft candidates \citep{tokenrecycling2025}, but a token match alone does not guarantee exact reuse at the next round's address.

A round has \emph{address coverage} when its verification pass has already computed a row at the next round's address. Figure~\ref{fig:roundmap} makes this condition concrete. Starting from context $c$, suppose verification accepts $a$ and then emits another $a$ as its continuation. The next round starts at $caa$, so it needs the row computed there. The row at $ca$ has a different address because it conditions on the context before the second $a$. Once retained, the matching row supplies exact first-token probabilities for selecting the next tree, provided the target computation is unchanged.

We find an exact trade-off between address coverage and accepted tokens when a pass scores only its candidate tree. The result holds for rounds that accept one tree path, emit one continuation token, and preserve the target distribution for each chosen tree. Coverage occurs exactly when the accepted path plus its continuation remains inside the tree. Exhaustive traversal instead keeps accepting tokens until a draw leaves the tree, then emits that draw as its continuation. The next root is therefore unscored. A covered round could have accepted its continuation as another draft token, so it accepts at least one fewer token than exhaustive traversal on the same target-token sequence.

Scoring prefixes one step beyond the candidate tree can cover the next root even after exhaustive traversal. Adding those same prefixes as candidates would instead let traversal accept one more token on each covered round.

We then compare coverage over whole runs. Within this verifier class, consider a decoder with one target pass in each round whose decisions use target information only along its emitted path. We construct a decoder that is never covered, scores no more prefixes and emits at least as many tokens after every round. Thus these decoders benefit from the next-root row by computing it in a separate target pass.

We evaluate the root row's value for tree selection separately from the cost of obtaining it. Our experiments use EAGLE-3 dynamic trees \citep{li2025eagle3} for six targets from 4B to 32B, seven task types and two decoding temperatures. Root selection ranks candidate paths using target probabilities for their first token and draft probabilities for later tokens. We find that it increases the expected number of accepted draft tokens over draft ranking, especially with small trees. Most gains occur when the models' most likely first tokens differ.

Root-draw selection instead draws the first token from the root row before allocating the tree beneath it. We find higher expected acceptance than pruning a preselected tree to that token's branch \citep{liu2024pearl}, even at equal mean tree size. Both selection rules also decode faster than draft ranking when standalone drafting on a second GPU overlaps the separate root pass.

\section{Related Work}
\label{sec:related}
\paragraph{Speculation and verification.}
Speculative sampling uses acceptance and correction steps to preserve the target distribution \citep{leviathan2023fast,chen2023accelerating}. Tree-based methods verify several candidate paths in parallel \citep{miao2024specinfer,cai2024medusa,li2024eagle2,li2025eagle3}. Related studies bound expected acceptance in draft trees \citep{speedoflight2025} and analyze how jointly verifying multiple tokens or drafts can improve acceptance while preserving the target distribution \citep{blockverify2024,optimalmdsd2025,univer2026}. We instead characterize when a verification pass has already computed a row at the next round's complete context and relate this availability to the number of accepted tokens.

\paragraph{Target information and reuse.}
LogitSpec predicts the token after the next one from last-token logits and uses it to retrieve drafts \citep{logitspec2025}. FlexDraft calibrates draft logits with the resolved continuation token \citep{flexdraft2026}, while Saguaro drafts ahead for the verification outcomes it predicts \citep{saguaro2026}. Token Recycling keeps the top candidates of every target row computed during verification under that row's input token and retrieves the next draft tree from this store \citep{tokenrecycling2025}. We evaluate full rows indexed by the final token as inputs to tree selection.

\paragraph{Drafting and compute allocation.}
Work on draft generation includes proposer training \citep{gto2025,vsd2026,lei2026draftopd,lee2026offpolicy}, parallel drafting \citep{hui2026peagle}, block drafting \citep{chen2026dflash,sandler2025specdiff2,liu2026dart,lee2026glance} and retrieval \citep{racer2026}. EAGLE-2, OPT-Tree, SMART and Sequoia adapt candidate trees to draft confidence or hardware cost \citep{li2024eagle2,opttree2024,smart2026,sequoia2024}. Beyond tree construction, agreement between two drafts can guide whether a reasoning model invokes extended thinking \citep{lee2026dart_routing}.

Pipelined methods overlap drafting with target verification \citep{liu2024pearl,pipespec2025}. For example, PEARL verifies the first draft token while the drafter proposes the rest, making the root row available during drafting. We isolate the value of target rows for tree selection by comparing rules on the same draft candidate pools.

\section{Setting and Reuse Criterion}
\label{sec:setting}
At round $t$, a target model $p$ and drafter $q$ share committed context $c_t$. The drafter selects a finite, prefix-closed candidate trie $T_t$ with at most $B$ non-root nodes. Each node represents the token sequence $u$ from the root, with full context $c_tu$. The root $\epsilon$ represents the empty sequence; siblings with the same token are merged.

One target pass scores a set $Q_t$ of prefixes, possibly outside $T_t$, and computes rows
\begin{equation}
 f_t(u)=p(\cdot\mid c_tu),\qquad u\in Q_t.
 \label{eq:field}
\end{equation}

Let $a_t$ be the accepted draft-token sequence and $b_t$ the single continuation token. The next committed context is $c_{t+1}=c_ta_tb_t$.

\begin{definition}[Exact address coverage]
\label{def:coverage}
A round is covered when $a_tb_t\in Q_t$. On this event, $f_t(a_tb_t)=p(\cdot\mid c_{t+1})$ is the next round's root row.
\end{definition}

Without coverage, the rows computed in this pass cannot guarantee exact recovery of the next root row for an unrestricted target. Two targets can agree on all these rows but differ at the next root.

For acceptance analysis, a \emph{path-aligned} verifier accepts only a root path, $a_t\in T_t$. \emph{Conditional exactness} requires the target law given the tree and prior history; marginal exactness requires it after averaging over tree selection.

\section{Coverage and Verification}
\label{sec:theory}
\subsection{Exhaustive traversal and early stopping}
\label{sec:route-byproduct}
For a fixed trie $T$, consider a target-distributed continuation $X_1,X_2,\ldots$ of context $c$. Let $L_T$ be the number of initial tokens that remain inside the trie. Its expectation is the sum of the target probabilities of the non-root paths,
\begin{equation}
 \begin{aligned}
 L_T&=\max\{d\geq0:X_{1:d}\in T\},\\
 \tau^*(T)&=\sum_{u\in T\setminus\{\epsilon\}}p(u\mid c).
 \end{aligned}
 \label{eq:ceiling}
\end{equation}
The path probability $p(u\mid c)$ is the product of target transitions along $u$. For greedy decoding, the target law is concentrated on the argmax path.

Exhaustive target sampling accepts target draws until one leaves the trie, then emits that draw as the continuation. It accepts $L_T$ tokens with expected acceptance $\tau^*(T)$ \citep{speedoflight2025}.

We compare a verifier with exhaustive traversal by extending its output with fresh target draws. When the verifier is conditionally exact given $T$, the resulting stream follows the target law. If it is also path-aligned, its $A$ accepted tokens form an initial path in $T$, so $A\leq L_T$; its continuation is $X_{A+1}$.

\begin{theorem}[Coverage and unused path length]
\label{thm:characterization}
Suppose the verifier is path-aligned and conditionally exact given $T$, exactly one continuation is emitted, and the scored set is the candidate trie, $Q=T$. Then coverage $C=\{X_{1:A+1}\in Q\}$ satisfies
\begin{equation}
 \begin{gathered}
 C=\{A<L_T\},\\
 \tau^*(T)-\mathbb E[A]
 =\mathbb E\!\left[(L_T-A)\mathbf1_C\right].
 \end{gathered}
 \label{eq:coverage-gap}
\end{equation}
Consequently, exhaustive target sampling has zero coverage. Positive coverage is equivalent to strictly smaller expected acceptance within this class.
\end{theorem}

On the same target stream, coverage means that the accepted path followed by the continuation remains in $T$. Exhaustive traversal would therefore accept at least one more token. Without coverage, both verifiers accept $L_T$ tokens. We thus obtain a lower bound on the expected acceptance gap:
\begin{equation}
 \mathbb E[L_T-A]\geq\Pr(C).
 \label{eq:coverage-cost}
\end{equation}
The full coupling proof is given in Appendix~\ref{app:identity}.

\paragraph{Scoring beyond the trie.}
We find that scoring outside $T$ can cover the next root even with exhaustive traversal. Let $F$ denote the scored frontier: scored prefixes outside $T$ whose parents lie in $T$. Coverage occurs exactly when the first exit prefix lies in $F$, giving
\begin{equation}
 \begin{aligned}
 \Pr(C)&=\sum_{u\in F}p(u\mid c)\\
 &=\tau^*(T\cup F)-\tau^*(T).
 \end{aligned}
 \label{eq:frontier}
\end{equation}
On the same target stream, exhaustive traversal of $T\cup F$ accepts one more token exactly on the covered rounds. Each such prefix could therefore extend the accepted path as a candidate. More generally, for any scored trie $Q\supseteq T$, the same exactness and stopping conditions give $\tau^*(Q)-\mathbb E[A]\geq\Pr(C)$. Proofs are given in Appendix~\ref{app:frontier-proof} and finite examples in Appendix~\ref{app:stopping}.

\paragraph{Rejection-based verifiers.}
We find that speculative sampling \citep{leviathan2023fast,chen2023accelerating} and its multi-step form for token trees \citep{miao2024specinfer} never cover the next root when $Q=T$. Their rejection rule suffices without conditional exactness given the tree. They test a node's children in turn; rejected tokens retain zero probability in every later residual distribution. The continuation therefore leaves the trie after all children are rejected, or at a leaf. Point-mass proposals, as in EAGLE-3, give exhaustive target sampling. Appendix~\ref{app:rejection} gives the residual argument.

\paragraph{Covering over a run.}
A covered round hands the next round its root row, so Equation~\ref{eq:coverage-cost} alone does not settle whether covering pays off over a run.
\begin{theorem}[Covering over a run]
\label{thm:run}
Let a decoder's only target computation after a shared prefill be one pass in each round $t$ over a trie $Q_t\supseteq T_t$, and let its verifier be path-aligned, emit one continuation and be conditionally exact given the selection inputs.

Suppose the decoder chooses $T_t$, $Q_t$ and where each round stops only from its emitted and drawn tokens and from the rows and other target outputs on its emitted path. Its only other inputs are its drafter, independent randomness and a first token drawn from its held root row.

Then some decoder preserves the target law, is never covered and scores no more prefixes per round. On a common target stream, it has emitted no fewer tokens than the original after each round, and more after the original's covered rounds.
\end{theorem}
We construct the comparator on the same target stream with shared independent randomness. It reconstructs the original tries but scores only its root and their prefixes extending its emitted context. Exhaustive traversal takes it past every stream prefix scored by the original. The comparator has scored these prefixes and drawn their next tokens. It thus holds the rows, hidden states and draws needed to reproduce the original's choices.

We extend the comparison to decoders using off-path target outputs. Obtaining them may require scoring branches diverging within the comparator's emitted context. The extension applies when the cache attention mask permits such scoring. Appendix~\ref{app:run} gives the construction and proof.

\subsection{Adaptive selection preserves exactness}
\label{sec:assembled}
Retained rows can guide tree selection without changing the target distribution. Let $\mathcal H_t$ contain the emitted history, retained target rows and drafter information after round $t$. The next tree may depend on $\mathcal H_t$, provided verification is conditionally exact given $\mathcal H_t$ and that tree. Each emitted token then follows the target conditional at its actual prefix. Multiplying these conditional probabilities gives the target joint distribution for any fixed number of output tokens.

Retained rows from other addresses can therefore guide tree selection, provided selection does not depend on the verifier's future random draws. The verifier must still preserve the target conditional at each emitted prefix. Directly sampling from a row at another address can change the output distribution, as illustrated in Appendix~\ref{app:wrong-branch}.

\section{Experiments}
\label{sec:results}
\begin{table*}[t]
\centering\fontsize{9}{9.5}\selectfont
\begin{tabular*}{\textwidth}{@{\extracolsep{\fill}}lrrrrrrrr@{}}
\toprule
& Exhaustive & \multicolumn{2}{c}{Cap 1} & \multicolumn{2}{c}{Sampled-leaf stop} & \multicolumn{3}{c}{Scored top $M$, $\Pr(C)$} \\
\cmidrule(lr){2-2}\cmidrule(lr){3-4}\cmidrule(lr){5-6}\cmidrule(lr){7-9}
Target & $\mathbb E[A]$ & $\Pr(C)$ & $\mathbb E[L_T-A]$ & $\Pr(C)$ & $\mathbb E[L_T-A]$ & $M=100$ & $M=200$ & $M=710$ \\
\midrule
\rowcolor{tablegroup}\multicolumn{9}{c}{Greedy decoding} \\
Llama-3.1-8B & 4.509 & 0.817 & 3.599 & 0.466 & 0.466 & 0.158 & 0.275 & 0.328 \\
Qwen3-4B & 2.340 & 0.701 & 1.452 & 0.241 & 0.241 & 0.128 & 0.261 & 0.345 \\
Qwen3-8B & 2.221 & 0.587 & 1.401 & 0.239 & 0.239 & 0.115 & 0.226 & 0.304 \\
Qwen2.5-7B & 3.293 & 0.686 & 2.459 & 0.293 & 0.293 & 0.123 & 0.225 & 0.285 \\
Qwen2.5-14B & 3.486 & 0.724 & 2.631 & 0.321 & 0.321 & 0.139 & 0.249 & 0.302 \\
Qwen2.5-32B & 3.489 & 0.720 & 2.633 & 0.316 & 0.316 & 0.133 & 0.240 & 0.297 \\
\rowcolor{tablegroup}\multicolumn{9}{c}{Sampling at temperature $1$} \\
Llama-3.1-8B & 3.528 & 0.698 & 2.695 & 0.363 & 0.363 & 0.143 & 0.261 & 0.321 \\
Qwen3-4B & 2.253 & 0.676 & 1.377 & 0.234 & 0.234 & 0.126 & 0.257 & 0.339 \\
Qwen3-8B & 2.082 & 0.553 & 1.286 & 0.223 & 0.223 & 0.112 & 0.219 & 0.295 \\
Qwen2.5-7B & 2.746 & 0.608 & 1.959 & 0.244 & 0.244 & 0.116 & 0.220 & 0.284 \\
Qwen2.5-14B & 2.895 & 0.645 & 2.075 & 0.267 & 0.267 & 0.128 & 0.239 & 0.303 \\
Qwen2.5-32B & 2.921 & 0.649 & 2.098 & 0.261 & 0.261 & 0.126 & 0.236 & 0.301 \\
\bottomrule
\end{tabular*}
\caption{Coverage, expected acceptance and shortfall on EAGLE-3's $59$-node trees, computed as expectations from the same recorded target rows for every rule. Scored sets keep exhaustive acceptance on the tree and add the next $M-59$ candidates in draft order.}
\label{tab:coverage}
\end{table*}

\subsection{Setup}
\label{sec:setup}
\paragraph{Targets and prompts.}
The six targets are Llama-3.1-8B-Instruct \citep{grattafiori2024llama3}, Qwen3-4B and Qwen3-8B \citep{qwen3technical2025}, and Qwen2.5-7B, 14B and 32B-Instruct \citep{qwen25technical2024}. Each uses a public EAGLE-3 drafter trained for it \citep{li2025eagle3,cen2026angelslim,pan2025fail} and its own chat template, with Qwen3 in non-thinking mode. We evaluate $164$ HumanEval problems \citep{chen2021evaluating} and $480$ Spec-Bench prompts \citep{xia2024unlocking}, including MT-bench \citep{zheng2023judging}. These cover seven task types. Each target decodes the $644$ prompts under a default limit of $256$ new tokens, with greedy decoding and untruncated sampling at temperature $1$.

\paragraph{Trees and verification.}
The EAGLE-3 drafter grows $710$ candidates at depths one to eight. Draft selection keeps the top $59$ non-root candidates by cumulative draft score, forming EAGLE-3's default $60$-token tree including the root. Verification uses exhaustive target sampling, which matches EAGLE-3's joint law for accepted tokens and the continuation.

To evaluate alternative trees, the pass also scores expanded nodes and the top $160$ candidates. Every candidate's parent is scored, supplying target probabilities for every pool edge. The verified tree's rows remain unchanged up to rounding. Checkpoints, tasks, hardware and the measurement protocol are detailed in Appendix~\ref{app:setup}.

\begin{table*}[t]
\centering\fontsize{9}{9.5}\selectfont
\setlength{\tabcolsep}{4pt}
\begin{tabular*}{\textwidth}{@{\extracolsep{\fill}}lrrrrrrrrr>{\columncolor{tableroot}[4pt][0pt]}r@{}}
\toprule
& \multicolumn{5}{c}{$B=8$} & \multicolumn{5}{c}{$B=59$} \\
\cmidrule(lr){2-6}\cmidrule(lr){7-11}
& \multicolumn{3}{c}{Expected acceptance} & \multicolumn{2}{c}{Root-draw gain, \%} & \multicolumn{3}{c}{Expected acceptance} & \multicolumn{2}{c}{Root-draw gain, \%} \\
\cmidrule(lr){2-4}\cmidrule(lr){5-6}\cmidrule(lr){7-9}\cmidrule(lr){10-11}
Target & Draft & \cellcolor{tableroot}Root & Oracle & \cellcolor{tableroot}Same cap & \cellcolor{tableroot}Equal mean & Draft & \cellcolor{tableroot}Root & Oracle & \cellcolor{tableroot}Same cap & Equal mean \\
\midrule
\rowcolor{tablegroup}\multicolumn{11}{c}{Greedy decoding} \\
Llama-3.1-8B & 3.186 & \cellcolor{tableroot}3.568 & 4.899 & \cellcolor{tableroot}10.0 & \cellcolor{tableroot}2.9 & 4.509 & \cellcolor{tableroot}4.676 & 4.899 & \cellcolor{tableroot}3.5 & 0.9 \\
Qwen3-4B & 1.664 & \cellcolor{tableroot}2.002 & 2.720 & \cellcolor{tableroot}12.8 & \cellcolor{tableroot}4.5 & 2.340 & \cellcolor{tableroot}2.516 & 2.720 & \cellcolor{tableroot}6.7 & 2.4 \\
Qwen3-8B & 1.556 & \cellcolor{tableroot}1.919 & 2.564 & \cellcolor{tableroot}13.8 & \cellcolor{tableroot}4.4 & 2.221 & \cellcolor{tableroot}2.399 & 2.564 & \cellcolor{tableroot}6.7 & 2.2 \\
Qwen2.5-7B & 2.359 & \cellcolor{tableroot}2.710 & 3.636 & \cellcolor{tableroot}10.7 & \cellcolor{tableroot}4.1 & 3.293 & \cellcolor{tableroot}3.450 & 3.636 & \cellcolor{tableroot}4.2 & 1.3 \\
Qwen2.5-14B & 2.435 & \cellcolor{tableroot}2.790 & 3.846 & \cellcolor{tableroot}11.1 & \cellcolor{tableroot}4.0 & 3.486 & \cellcolor{tableroot}3.645 & 3.846 & \cellcolor{tableroot}4.2 & 1.2 \\
Qwen2.5-32B & 2.464 & \cellcolor{tableroot}2.812 & 3.843 & \cellcolor{tableroot}10.6 & \cellcolor{tableroot}3.9 & 3.489 & \cellcolor{tableroot}3.646 & 3.843 & \cellcolor{tableroot}4.1 & 1.2 \\
\rowcolor{tablegroup}\multicolumn{11}{c}{Sampling at temperature $1$} \\
Llama-3.1-8B & 2.444 & \cellcolor{tableroot}2.655 & 3.225 & \cellcolor{tableroot}12.2 & \cellcolor{tableroot}3.1 & 3.528 & \cellcolor{tableroot}3.634 & 3.882 & \cellcolor{tableroot}4.7 & 1.4 \\
Qwen3-4B & 1.600 & \cellcolor{tableroot}1.895 & 2.595 & \cellcolor{tableroot}13.1 & \cellcolor{tableroot}4.7 & 2.253 & \cellcolor{tableroot}2.411 & 2.627 & \cellcolor{tableroot}6.8 & 2.4 \\
Qwen3-8B & 1.449 & \cellcolor{tableroot}1.762 & 2.356 & \cellcolor{tableroot}14.2 & \cellcolor{tableroot}4.8 & 2.082 & \cellcolor{tableroot}2.241 & 2.416 & \cellcolor{tableroot}7.0 & 2.3 \\
Qwen2.5-7B & 1.922 & \cellcolor{tableroot}2.161 & 2.704 & \cellcolor{tableroot}12.6 & \cellcolor{tableroot}4.8 & 2.746 & \cellcolor{tableroot}2.866 & 3.079 & \cellcolor{tableroot}5.2 & 1.8 \\
Qwen2.5-14B & 1.996 & \cellcolor{tableroot}2.221 & 2.785 & \cellcolor{tableroot}13.0 & \cellcolor{tableroot}4.7 & 2.895 & \cellcolor{tableroot}3.010 & 3.237 & \cellcolor{tableroot}5.3 & 1.8 \\
Qwen2.5-32B & 2.039 & \cellcolor{tableroot}2.256 & 2.816 & \cellcolor{tableroot}12.2 & \cellcolor{tableroot}4.5 & 2.921 & \cellcolor{tableroot}3.033 & 3.263 & \cellcolor{tableroot}5.2 & 1.8 \\
\bottomrule
\end{tabular*}
\caption{Expected accepted draft tokens and root-draw gains over the PEARL-style decoder at shared contexts. Same cap uses the stated $B$; Equal mean enlarges PEARL-style preselection to match the mean verified node count.}
\label{tab:selection}
\end{table*}

\begin{figure*}[t]
\centering
\includegraphics[width=\textwidth]{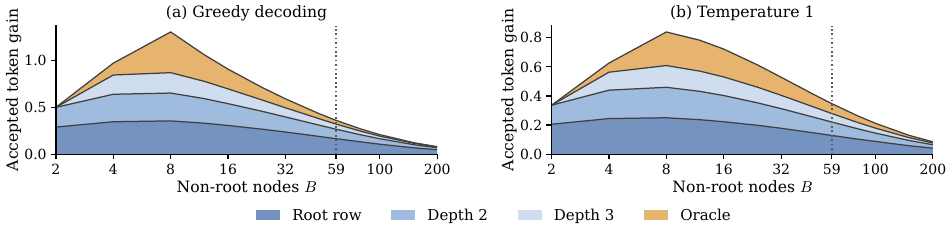}
\caption{Gain over draft selection when exact target rows replace draft factors through increasing depths, averaged over six targets. The upper edge of each band is one selection rule. The dotted line marks EAGLE-3's tree.}
\label{fig:depth}
\end{figure*}

\begin{figure*}[t]
\centering
\includegraphics[width=\textwidth]{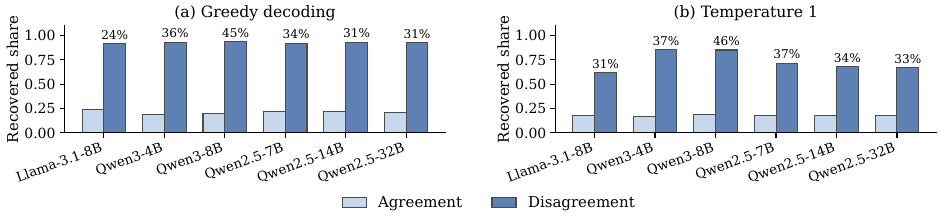}
\caption{Share of the oracle advantage at $B=59$ recovered by root selection, split by whether the drafter's most probable first token is the target's most probable token. Labels give the share of rounds in which the two differ.}
\label{fig:agreement}
\end{figure*}

\paragraph{Selection values.}
We compare rules at the same recorded contexts and candidate pools. For trees chosen before the first draw, each rule keeps its $B$ highest-ranked non-root candidates, breaking ties by depth and then draft rank. This set $S$ contains their non-root ancestors. Its expected number of accepted draft tokens under exhaustive traversal is
\begin{equation}
 V(S)=\sum_{u\in S}p(u\mid c).
 \label{eq:selection-value}
\end{equation}
Under greedy decoding, $p$ places all mass on the target argmax at each prefix, in both $V$ and scores using $p$. Thus $V$ counts selected nodes on the greedy path.

For $u=(x_1,\ldots,x_d)$, root selection replaces only the first draft factor with the exact root probability,
\begin{equation}
 w_{\mathrm{root}}(u)=p(x_1\mid c)\prod_{i=2}^{d}q_i,
 \label{eq:root-score}
\end{equation}
where $q_i$ is the drafter's probability of $x_i$. Depth-two selection also uses the exact second factor. The oracle ranks by $p(u\mid c)$ and maximizes $V$ among trees fixed before verification \citep{opttree2024,sequoia2024}.

Root-draw selection draws the first token from the root row, then selects it and up to $B-1$ descendants in draft order. In contrast, a PEARL-style decoder \citep{liu2024pearl} prunes draft selection's preselected tree after the same draw. It keeps the drawn token's subtree, or that token alone if absent from the tree's first level. If the pool's first level lacks the token, both verify it alone.

Both decoders accept the first draw; their expected acceptance includes that token and averages over the root draw. Under greedy decoding, root-draw and root selection have equal expected acceptance when the pool's first level contains the target argmax. Reported $95\%$ percentile intervals use $2{,}000$ paired prompt-bootstrap resamples.

\subsection{Coverage on EAGLE-3 trees}
\label{sec:finitechecks}

Table~\ref{tab:coverage} compares stopping rules and scored sets. We find no coverage under exhaustive traversal with only the tree scored. Cap~1 accepts at most one draft token; sampled-leaf stopping emits a draw reaching a leaf as its continuation, accepting one fewer token than exhaustive traversal on covered rounds.

Extra scoring preserves the original tree's exhaustive acceptance. Coverage requires a scored prefix just outside it, which could instead extend the accepted path by one token.

\begin{table}[t]
\centering\fontsize{9}{9.5}\selectfont
\setlength{\tabcolsep}{2pt}
\begin{tabular*}{\columnwidth}{@{\extracolsep{\fill}}lrrrr@{}}
\toprule
& \multicolumn{2}{c}{Greedy decoding} & \multicolumn{2}{c}{Temperature $1$} \\
\cmidrule(lr){2-3}\cmidrule(lr){4-5}
Target & \cellcolor{tableroot}Gain, \% & Nodes & \cellcolor{tableroot}Gain, \% & Nodes \\
\midrule
Llama-3.1-8B & \cellcolor{tableroot}3.3 [2.6, 4.1] & 53 / 39 & \cellcolor{tableroot}4.1 [1.7, 6.8] & 48 / 34 \\
Qwen3-4B & \cellcolor{tableroot}6.7 [6.2, 7.3] & 49 / 33 & \cellcolor{tableroot}7.1 [6.4, 7.9] & 48 / 32 \\
Qwen3-8B & \cellcolor{tableroot}7.3 [6.6, 8.0] & 44 / 28 & \cellcolor{tableroot}7.3 [6.3, 8.3] & 43 / 27 \\
Qwen2.5-7B & \cellcolor{tableroot}4.5 [3.6, 5.5] & 49 / 35 & \cellcolor{tableroot}5.7 [4.1, 7.4] & 46 / 32 \\
Qwen2.5-14B & \cellcolor{tableroot}3.8 [3.1, 4.4] & 49 / 36 & \cellcolor{tableroot}4.5 [3.0, 5.9] & 47 / 32 \\
Qwen2.5-32B & \cellcolor{tableroot}3.7 [3.0, 4.3] & 50 / 36 & \cellcolor{tableroot}4.7 [3.4, 6.1] & 47 / 33 \\
\bottomrule
\end{tabular*}
\caption{Online decoding at $B=59$ with a separate root pass. Gain is root-draw selection's gain in accepted draft tokens over the PEARL-style decoder, with $95\%$ intervals. Nodes gives mean non-root counts as root draw / PEARL-style at each decoder's own contexts.}
\label{tab:online-main}
\end{table}

\begin{table*}[t]
\centering\fontsize{9}{9.5}\selectfont
\setlength{\tabcolsep}{4pt}
\begin{tabular*}{\textwidth}{@{\extracolsep{\fill}}lrrrrrrr>{\columncolor{tableroot}[4pt][0pt]}r@{}}
\toprule
& \multicolumn{4}{c}{$B=8$} & \multicolumn{4}{c}{$B=59$} \\
\cmidrule(lr){2-5}\cmidrule(lr){6-9}
Target & Draft & Base & Rows & \cellcolor{tableroot}Root & Draft & Base & Rows & Root \\
\midrule
\rowcolor{tablegroup}\multicolumn{9}{c}{Greedy decoding} \\
Llama-3.1-8B & 3.186 & 3.187 & 3.234 & \cellcolor{tableroot}3.568 & 4.509 & 4.509 & 4.536 & 4.676 \\
Qwen3-4B & 1.664 & 1.674 & 1.722 & \cellcolor{tableroot}2.002 & 2.340 & 2.343 & 2.372 & 2.516 \\
Qwen3-8B & 1.556 & 1.581 & 1.650 & \cellcolor{tableroot}1.919 & 2.221 & 2.241 & 2.272 & 2.399 \\
Qwen2.5-7B & 2.359 & 2.360 & 2.399 & \cellcolor{tableroot}2.710 & 3.293 & 3.294 & 3.311 & 3.450 \\
Qwen2.5-14B & 2.435 & 2.436 & 2.475 & \cellcolor{tableroot}2.790 & 3.486 & 3.482 & 3.502 & 3.645 \\
Qwen2.5-32B & 2.464 & 2.458 & 2.490 & \cellcolor{tableroot}2.812 & 3.489 & 3.488 & 3.502 & 3.646 \\
\rowcolor{tablegroup}\multicolumn{9}{c}{Sampling at temperature $1$} \\
Llama-3.1-8B & 2.444 & 2.451 & 2.482 & \cellcolor{tableroot}2.655 & 3.528 & 3.527 & 3.544 & 3.634 \\
Qwen3-4B & 1.600 & 1.613 & 1.656 & \cellcolor{tableroot}1.895 & 2.253 & 2.257 & 2.283 & 2.411 \\
Qwen3-8B & 1.449 & 1.475 & 1.537 & \cellcolor{tableroot}1.762 & 2.082 & 2.105 & 2.133 & 2.241 \\
Qwen2.5-7B & 1.922 & 1.931 & 1.958 & \cellcolor{tableroot}2.161 & 2.746 & 2.755 & 2.768 & 2.866 \\
Qwen2.5-14B & 1.996 & 2.002 & 2.025 & \cellcolor{tableroot}2.221 & 2.895 & 2.901 & 2.914 & 3.010 \\
Qwen2.5-32B & 2.039 & 2.045 & 2.066 & \cellcolor{tableroot}2.256 & 2.921 & 2.924 & 2.937 & 3.033 \\
\bottomrule
\end{tabular*}
\caption{Expected accepted draft tokens at shared contexts for draft selection, cross-fitted rerankers and root selection. Base uses draft and contextual features; Rows adds features from retained rows at other addresses.}
\label{tab:rerank-main}
\end{table*}

\begin{table}[t]
\centering\footnotesize
\setlength{\tabcolsep}{2.5pt}
\begin{tabular*}{\columnwidth}{@{\extracolsep{\fill}}llrr@{}}
\toprule
Target & Decoding & $B=8$ & $B=59$ \\
\midrule
Llama-3.1-8B & Greedy & 8.5 [7.3, 9.8] & 2.7 [2.4, 3.1] \\
Qwen3-8B & Greedy & 22.6 [20.9, 24.1] & 7.6 [7.2, 8.1] \\
\midrule
Llama-3.1-8B & Temp. $1$ & 7.6 [6.9, 8.4] & 2.8 [2.6, 3.1] \\
Qwen3-8B & Temp. $1$ & 21.4 [20.4, 22.5] & 7.3 [7.0, 7.6] \\
\bottomrule
\end{tabular*}
\caption{Root selection's expected-acceptance gain over draft ranking under a $1{,}024$-token limit, in percent, for rounds starting once the default output limit is reached. Brackets give $95\%$ prompt-bootstrap intervals.}
\label{tab:long-main}
\end{table}

\subsection{Rows at matched addresses}
\label{sec:rows}

Additional scoring provides rows at some complete emitted addresses. We find that these rows closely match the next root row: mean total variation is below $0.01$ in every run. By contrast, rows sharing only the final token have mean total variation between $0.38$ and $0.47$. The remaining variation for address-matched rows is comparable to evaluating the same row in batches of different sizes.

Frequent token matches in the previous tree therefore do not guarantee an exact row for the next root. Address counts and row agreement are detailed in Appendix~\ref{app:matches}.

\subsection{Value of exact rows for tree selection}
\label{sec:selection}

Table~\ref{tab:selection} compares selection rules at shared contexts. We find that root selection increases expected acceptance over draft ranking by up to $8.0\%$ at the default budget of $59$ nodes, with larger gains at $8$ nodes. Under greedy decoding, root selection cannot shorten the accepted path relative to draft selection.

Figure~\ref{fig:depth} compares gains over draft ranking from exact target probabilities at increasing depths. We recover a substantial share of the oracle's gain with the root row alone, while exact second-token probabilities further improve allocation within first-token branches. Our results show that exact target probabilities are especially useful for small trees, since draft ranking already retains many likely target paths in larger trees.

\paragraph{First-token agreement.}
\label{sec:mechanism}
Figure~\ref{fig:agreement} groups rounds by whether the draft and target have the same most probable first token and compares root selection's recovery of the oracle advantage. We find that the less frequent disagreement rounds account for most of the gain. Root selection rescales candidate scores within each first-token branch by a common factor, shifting node allocation across branches while leaving later-token draft probabilities unchanged. The detailed comparison is given in Appendix~\ref{app:agreement}.

\begin{table*}[t]
\centering\footnotesize
\setlength{\tabcolsep}{2.5pt}
\begin{tabular*}{\textwidth}{@{\extracolsep{\fill}}lrrrrrrr>{\columncolor{tableroot}[2.5pt][0pt]}r@{}}
\toprule
Target & MT-bench & Translation & Summarization & QA & Math & RAG & HumanEval & All \\
\midrule
Llama-3.1-8B & 4.9 / 3.7 & 3.9 / 3.6 & 4.4 / 3.8 & 6.1 / 4.8 & 3.7 / 3.6 & 4.0 / 3.2 & 2.3 / 1.9 & 3.7 / 3.0 \\
Qwen3-4B & 9.1 / 7.9 & 12.0 / 11.5 & 6.9 / 6.5 & 9.8 / 9.4 & 7.0 / 6.9 & 7.6 / 7.0 & 6.4 / 6.0 & 7.5 / 7.0 \\
Qwen3-8B & 8.8 / 8.9 & 8.7 / 8.6 & 7.1 / 7.1 & 10.2 / 9.4 & 8.3 / 8.0 & 8.2 / 7.5 & 7.2 / 6.8 & 8.0 / 7.7 \\
Qwen2.5-7B & 5.4 / 4.7 & 6.2 / 5.8 & 5.4 / 4.9 & 7.5 / 5.9 & 4.6 / 4.4 & 5.3 / 4.6 & 3.4 / 3.5 & 4.8 / 4.4 \\
Qwen2.5-14B & 5.7 / 4.5 & 5.6 / 5.5 & 5.2 / 4.4 & 7.2 / 5.7 & 3.6 / 3.5 & 5.4 / 4.5 & 3.4 / 3.2 & 4.6 / 4.0 \\
Qwen2.5-32B & 5.3 / 4.1 & 6.3 / 5.5 & 5.3 / 4.5 & 7.2 / 5.6 & 3.8 / 3.0 & 5.9 / 4.4 & 3.1 / 3.1 & 4.5 / 3.8 \\
\bottomrule
\end{tabular*}
\caption{Root selection's expected-acceptance gain over draft ranking at $B=59$, in percent. Each cell gives greedy decoding / temperature $1$. All aggregates the complete prompt set.}
\label{tab:by-task}
\end{table*}

\paragraph{Drawing the first token.}
Table~\ref{tab:selection} compares root-draw selection and PEARL-style pruning with the same root row and first draw. Under the same cap, pruning discards other first-token branches and retains fewer nodes on average.

We therefore enlarge the baseline's preselected tree to match mean verified sizes. At a $59$-node root-draw cap, we find up to $2.4\%$ higher expected acceptance than the enlarged baseline. This baseline averages more verified nodes than root-draw selection when the drafter's most probable first token is drawn, but fewer for other candidate first tokens. Matching details and intervals are in Appendix~\ref{app:selci}.

\begin{figure}[t]
\centering
\includegraphics[width=\columnwidth]{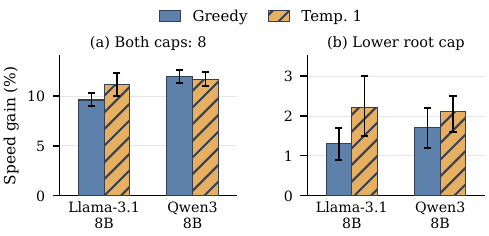}
\caption{Root-draw selection's decoding-speed gain over the PEARL-style decoder on two RTX A6000 GPUs. Left: both caps are $8$. Right: the PEARL-style cap is $59$ and root-draw caps are $40$--$43$.}
\label{fig:speed}
\end{figure}

\paragraph{Online decoding.}
Table~\ref{tab:online-main} compares the mean number of draft tokens accepted across rounds at each decoder's own contexts. We find that root-draw selection exceeds the PEARL-style decoder on all six targets at both temperatures, consistent with the shared-context comparison. Separately, root-draw selection retains higher expected acceptance when both trees are evaluated at its contexts with equal mean sizes. Protocols and additional results appear in Appendix~\ref{app:online}.

\paragraph{Standalone drafters.}
Figure~\ref{fig:speed} shows root-draw selection's speed gains over PEARL-style decoding with standalone drafters. Llama-3.2-1B-Instruct \citep{meta2024llama32} drafts for Llama-3.1-8B and Qwen3-0.6B \citep{qwen3technical2025} for Qwen3-8B. On two RTX A6000 GPUs, the root pass finishes during drafting. Root and root-draw selection also decode faster than draft selection.

To bring mean verified tree sizes closer, we lower root-draw selection's cap to $40$--$43$ and keep the PEARL-style cap at $59$. We find that root-draw selection decodes up to $2.2\%$ faster than the PEARL-style decoder while verifying at least as many nodes on average. These results show that allocating the tree after drawing the first token can improve decoding speed at similar mean tree sizes. Timing and acceptance details appear in Appendix~\ref{app:pipe}.

\paragraph{Longer outputs.}
Table~\ref{tab:long-main} tests whether root selection's gains persist beyond the default output limit. We extend the limit to $1{,}024$ tokens for Llama-3.1-8B and Qwen3-8B and evaluate rounds starting once the default output limit is reached. We find higher expected acceptance than draft ranking at both temperatures and budgets, with all $95\%$ intervals above zero. These gains show that exact root probabilities remain useful for tree selection in later decoding rounds. Counts and whole-output comparisons appear in Appendix~\ref{app:long}.

\paragraph{Task coverage.}
Table~\ref{tab:by-task} reports root selection's expected-acceptance gains over draft ranking by task. We find gains for every evaluated target and task at both temperatures. Since tasks and prompts contribute different numbers of decoding rounds, we also compute the overall gain with equal task weights or equal prompt weights. The gain remains positive under both alternatives. Task-stratified intervals and weighting details are given in Appendix~\ref{app:taskweights}.

\subsection{Rows at other addresses}
\label{sec:approx}

Table~\ref{tab:rerank-main} tests whether rows from other addresses can improve tree selection when used as features. We compare two rerankers: Base predicts target probabilities for first-token candidates from draft and contextual features, while Rows adds features from retained target rows. Both rerankers replace only the first draft-probability factor with their predictions. We train on one half of the prompts and evaluate on the other, then swap their roles.

We find that Rows improves expected acceptance over both draft ranking and Base for every evaluated target at both temperatures and budgets, with all paired $95\%$ intervals above zero. Root selection yields higher expected acceptance than either reranker. By contrast, directly substituting a retained row for the first draft factor lowers expected acceptance in these tests. Thus, rows from other addresses can guide tree allocation through learned predictions without being used as exact next-token distributions. Fitting details and direct-substitution tests are given in Appendices~\ref{app:calib} and~\ref{app:approx}, respectively.

\section{Conclusion}
\label{sec:conclusion}
We show that covering the next root costs accepted tokens when verification scores only its candidate tree. The result holds for verifiers that accept one path, emit one continuation and preserve the target distribution conditional on the chosen tree. For any such decoder using one target pass in each round and target information only along the emitted path, we construct an exact decoder that never covers the next root. Compared with the original decoder, it scores no more prefixes in each round and has emitted at least as many tokens after every round.

We also find that a separately computed root row increases expected acceptance over draft ranking, especially when the draft and target favor different first tokens. Drawing the first token before allocating the tree improves expected acceptance over pruning a preselected tree, even at equal mean verified tree size. In our measured deployment, the separate pass overlaps a standalone drafter on a second GPU, and both selection rules decode faster than draft selection. These findings motivate evaluating the cost of obtaining target rows together with the tree-selection gains they enable.

\FloatBarrier
\label{page:main-end}
\clearpage

\section*{Limitations}
\label{sec:limitations}
Our acceptance study covers six targets with public EAGLE-3 drafters and recorded target rows. The timing results apply to batch-one decoding with a standalone drafter on a second RTX A6000 GPU, where drafting outlasts the separate target pass.

The coverage identity concerns path-aligned, conditionally exact verification with one continuation and identical scored and candidate sets. The run-level comparison uses target information on the emitted path; its extension to off-path outputs under a suitable cache attention mask is described in Appendix~\ref{app:run}.

\section*{Ethics Statement}
The experiments analyze model probabilities on public benchmark prompts from Spec-Bench and HumanEval. The models and datasets retain their licenses and usage restrictions. The selection rules we study keep the target's output distribution and change only how many tokens a round emits.

\bibliography{references_arxiv}
\appendix
\raggedbottom
\makeatletter
\setlength{\@dblfptop}{0pt}
\setlength{\@dblfpsep}{12pt plus 2pt}
\setlength{\@dblfpbot}{0pt plus 1fil}
\makeatother
\setcounter{dbltopnumber}{3}
\renewcommand{\dbltopfraction}{0.85}

\section{Proofs and Scope}
\label{app:proofs}

\subsection{Cache-only recovery}
Fix a finite set of queried prefix addresses $Q$ and a requested address $z\notin Q$. Define two autoregressive targets $p_0,p_1$ that assign the same next-token row at every queried address, but different rows at $z$. Such targets exist because the class imposes no relation between transition rows at different prefixes. They produce identical row-query transcripts. A deterministic recovery function must return the same answer on both transcripts, so it cannot equal both $p_0(\cdot\mid z)$ and $p_1(\cdot\mid z)$. Randomization cannot provide exact recovery with probability one for both targets either.

The query model imposes no structure linking rows. Known finite-state structure, a context-independent distribution or an exactly matching drafter can determine a missing row using additional information. A learned approximation can also improve selection without exact recovery.

\subsection{Fixed-tree ceiling and coverage identity}
\label{app:identity}
Fix the prior context and candidate trie $T$. Extend the $A+1$ tokens that the round emits with fresh target draws at the growing context. Conditional exactness given $T$ makes every emitted token target-distributed given the tokens before it. The fresh draws have the same property by construction.

The extended sequence is therefore an infinite target-distributed stream $X$ with $T$ held fixed. Its first $A$ tokens are accepted, and $X_{A+1}$ is the continuation. Exactness must hold given the realized tree because $L_T$ is computed on that tree. Path alignment gives $A\leq L_T$ pointwise. The tail-sum identity now gives
\begin{equation}
 \begin{aligned}
 \mathbb E[L_T]
 &=\sum_{d\geq1}\Pr(X_{1:d}\in T)\\
 &=\sum_{u\in T\setminus\{\epsilon\}}p(u\mid c)=\tau^*(T).
 \end{aligned}
 \label{eq:tail-sum}
\end{equation}
At each depth, distinct trie paths represent disjoint events. The trie is finite, so the sum is finite. Exhaustive target sampling accepts the maximal initial path and attains this expectation.

When $Q=T$, the continuation is covered exactly when $X_{1:A+1}$ remains inside $T$. Prefix closure makes this equivalent to $A<L_T$. Thus $L_T-A$ vanishes outside coverage, proving Equation~\ref{eq:coverage-gap}. Since the difference is a nonnegative integer, its expectation is positive exactly when coverage has positive probability.

Besides $Q=T$, the proof uses three properties of the verifier, namely path alignment, conditional exactness given the tree and a single continuation. Appendix~\ref{app:setup} shows that EAGLE-3's tree verifier has all three.

\subsection{Scored frontier}
\label{app:frontier-proof}
Let the scored set $Q$ strictly contain $T$, and let $F$ contain the prefixes in $Q\setminus T$ whose parents lie in $T$. Exhaustive traversal on $T$ accepts $X_{1:L_T}$ and emits the address $X_{1:L_T+1}$. This address lies outside $T$, but its parent lies in $T$. The round is therefore covered exactly when the address belongs to $F$.

Conversely, if $X_{1:|u|}=u$ for some $u\in F$, every proper prefix of $u$ lies in $T$, so traversal exits exactly at $u$. Two distinct frontier elements cannot both be prefixes of $X$: the shorter would be a prefix of the longer one's parent and would then lie in $T$. These exit events are disjoint, giving
\begin{equation}
 \begin{aligned}
 \Pr(C)&=\sum_{u\in F}\Pr(X_{1:|u|}=u)\\
 &=\sum_{u\in F}p(u\mid c).
 \end{aligned}
 \label{eq:frontier-probability}
\end{equation}
Every element of $F$ has its parent in $T$, so $T\cup F$ is a trie. Its newly added nodes have no children in $T\cup F$. On the same stream, exhaustive traversal of this enlarged trie therefore accepts one additional token exactly on coverage:
\begin{equation}
 \begin{aligned}
 L_{T\cup F}&=L_T+\mathbf1_C,\\
 \tau^*(T\cup F)-\tau^*(T)&=\Pr(C).
 \end{aligned}
 \label{eq:frontier-extension}
\end{equation}
The second line takes expectations of the first using the tail-sum identity.

Now let a verifier with the three properties accept $A$ tokens on $T$, and let $Q\supseteq T$ be a scored trie. On coverage, $X_{1:A+1}\in Q$, so $L_Q\geq A+1$. Otherwise, $L_Q\geq L_T\geq A$. Hence
\begin{equation}
 \begin{aligned}
 L_Q-A&\geq\mathbf1_C,\\
 \tau^*(Q)-\mathbb E[A]&\geq\Pr(C).
 \end{aligned}
 \label{eq:scored-trie-bound}
\end{equation}
Equality holds for exhaustive traversal when $Q=T\cup F$.

\subsection{Rejection-based verifiers}
\label{app:rejection}
Speculative sampling and its multi-step form for token trees try a node's children in turn. A child $x$ proposed by $q$ is accepted with probability $\min\{1,r(x)/q(x)\}$, where $r$ starts as the node's target row. Each rejection replaces $r$ with the normalized $\max\{r-q,0\}$.

Rejecting $x$ requires $r(x)<q(x)$, so $x$ has zero mass under every later residual row. The continuation drawn after the last rejection is therefore none of the node's children, and a continuation drawn at a leaf also leaves the trie. When the pass scores only the trie, the next root is thus never scored, although these verifiers need not be conditionally exact given the tree.

\subsection{Adaptive selection}
\label{app:adaptive}
Let $c$ be the initial context and $X_1,X_2,\ldots$ the stream emitted across rounds. Condition on the full emitted history before token $j$ and on the past information used to select the current tree. The tree is a function of that information. Thus conditional exactness given the information and the tree gives
\begin{equation}
 \begin{aligned}
 &\Pr(X_j=x\mid X_{<j},\text{past information})\\
 &\qquad=p(x\mid cX_{<j}).
 \end{aligned}
 \label{eq:adaptive-law}
\end{equation}
Taking conditional expectations and multiplying these token-level factors yields the target autoregressive law at any fixed horizon. The grouping into rounds can depend on token values without changing this factorization.

\subsection{A row at the wrong branch}
\label{app:wrong-branch}
Consider a tree whose root has children $a$ and $b$. Let the target be deterministic with transitions $c\mapsto a$, $ca\mapsto b$, $cab\mapsto a$, and $cb\mapsto b$. A rule that verifies the siblings $a,b$ as one sequence can accept both by consuming rows at $c$ and $ca$, then take its continuation from the row scored at $cb$. It emits $abb$. Autoregressive decoding emits $aba$, because the third row must be conditioned on $cab$. Exactness of every scored row at its own address therefore does not ensure exactness of a sequence assembled from incorrect addresses.

\subsection{Covering over a run}
\label{app:run}
\begin{proof}[Proof of Theorem~\ref{thm:run}]
\emph{Common target stream.}
Every scored row is exact. Extend each round's output with fresh target draws as above. The result is a target stream whose first $A+1$ tokens are the round's output. When the first token is drawn before selection, it comes from the held root row and the verifier is exact given that draw, so the same construction applies.

To realize this coupling, draw target tokens in order. At each $X_j$, declare it the continuation with the round's conditional stopping probability given $X_{1:j}$, the tree, the history and survival to that position. This decision does not depend on later draws, which are fresh target draws after a stop. With randomness independent of the stream $X$, the first decoder therefore emits $X_{1:n_t}$ after round $t$, starting from $n_0=0$ and the shared prefill context $c$.

\emph{Comparator construction.}
Let the second decoder emit $X_{1:n'_t}$, with $n'_0=0$. Assume inductively that $n'_t\geq n_t$. The extra tokens it has already emitted form the string $u_t=X_{n_t+1:n'_t}$, which is empty when the lengths agree.

Using the same independent randomness, the second decoder computes the first decoder's tries $T_{t+1}$ and $Q_{t+1}$ at $cX_{1:n_t}$. It then scores and verifies the trie
\begin{equation}
 \begin{aligned}
 Q'_{t+1}=\{\epsilon\}\cup
 \{w\ne\epsilon:u_tw\in Q_{t+1}\}.
 \end{aligned}
 \label{eq:run-subtrie}
\end{equation}
These are the root and the scored prefixes that extend its own emitted context. It traverses $Q'_{t+1}$ exhaustively, taking draws from $X$. Its scored and candidate tries coincide, so Theorem~\ref{thm:characterization} gives zero coverage.

\emph{Emitted lengths and query count.}
Let $L$ be the largest $d$ with $X_{n_t+1:n_t+d}\in Q_{t+1}$. The first decoder accepts only within $T_{t+1}\subseteq Q_{t+1}$, so its accepted length satisfies $A\leq L$ and
\begin{equation}
 n_{t+1}=n_t+A+1\leq n_t+L+1.
 \label{eq:run-first-length}
\end{equation}
If its round is covered, the continuation also lies in $Q_{t+1}$, giving the stronger bound $n_{t+1}\leq n_t+L$.

Given $n'_t\geq n_t$, every remaining stream prefix through position $n_t+L$ lies in $Q'_{t+1}$. The next token exits $Q_{t+1}$. If $n'_t\leq n_t+L$, exhaustive traversal therefore emits through $X_{n_t+L+1}$. Otherwise it emits one token. Consequently,
\begin{equation}
 \begin{aligned}
 n'_{t+1}&=\max(n'_t,\,n_t+L)+1\\
 &\geq n_t+L+1.
 \end{aligned}
 \label{eq:run-advance}
\end{equation}
This is at least $n_{t+1}$ and strictly larger after a covered round, completing the length induction.

The map $w\mapsto u_tw$ injects the nonempty strings of $Q'_{t+1}$ into $Q_{t+1}\setminus\{\epsilon\}$. Including the root, the comparator scores at most $|Q_{t+1}|$ prefixes. When $Q_{t+1}=T_{t+1}$, it also has at most $B$ non-root nodes.

\emph{Available information.}
It remains to check, in the same induction, that the comparator uses only its own past. After round $t$, it has scored every prefix $X_{1:j}$ with $j<n'_t$, as a tree node or a round's root. If the first decoder scored $X_{1:j}$ in round $s\leq t$, then $X_{n_{s-1}+1:j}\in Q_s$. Thus $j-n_{s-1}$ is at most that round's $L$, and Equation~\ref{eq:run-advance} gives $n'_t\geq n'_s>j$.

The comparator therefore holds every target output that the first decoder holds on $X$. Whenever the first decoder holds its root row, the comparator has also emitted the first token drawn from that row. For a root row retained by the shared prefill, both decoders hold it initially; the comparator draws $X_1$ as the first step of its exhaustive traversal whenever the first decoder does.

Since $n_{t+1}\leq n'_{t+1}$, the first decoder's round ends within the comparator's emitted stream. The comparator can therefore replay its stopping decisions using tokens and target outputs already held, together with the same independent randomness. Its own traversal uses exact target conditionals, so the token factorization above preserves the target law.
\end{proof}

\paragraph{Off-path target outputs.}
Now allow the first decoder to read outputs at scored prefixes off $X$, and allow a pass to score prefixes that branch before the committed context ends. For such a prefix, the attention mask gives its new tokens their positions in that prefix and hides later cached positions. The returned rows remain exact because cached keys and values depend only on earlier tokens.

Keep the notation $u_t=X_{n_t+1:n'_t}$. In round $t+1$, the comparator additionally scores $cX_{1:n_t}v$ for every $v\in Q_{t+1}$ that is neither a prefix nor an extension of $u_t$. Each such $v$ is nonempty and lies outside the image of $w\mapsto u_tw$. Thus the total still does not exceed $|Q_{t+1}|$. These extra prefixes leave $X$, whereas the comparator's next context lies on $X$ and outside $Q'_{t+1}$, so it remains uncovered.

An off-path prefix $cX_{1:n_t}v$ cannot have $v$ as a prefix of $u_t$, since that would place it on $X$. If $v$ extends $u_t$, the comparator scores it through $Q'_{t+1}$; otherwise it is one of the additional prefixes. The comparator therefore holds every off-path target output available to the first decoder as well.

If the first decoder also scores prefixes that branch before its committed context ends, the comparator scores each of them too. They leave $X$, so the query count, zero coverage and availability of target outputs follow in the same way. The remainder of the coupling is unchanged.

\section{Stopping Policies and Finite Examples}
\label{app:fixture}
\subsection{Stopping placement and row cost}
\label{app:stopping}
Consider an IID uniform binary target on
\begin{equation}
 \begin{aligned}
 T\setminus\{\epsilon\}=\{&a,aa,aaa,aab,\\
 &ab,b,ba\},
 \end{aligned}
 \label{eq:binary-trie}
\end{equation}
with $Q=T$ and eight scored rows including the root. A cap $k$ stops at the first tree exit or after $k$ accepted tokens and emits the next unaccepted draw. Sampled-leaf stopping emits the current draw when it labels a leaf child, otherwise following matching children until exit. The decision uses only the current draw and trie topology.

Every rule scores the same eight rows before stopping. Exhaustive traversal has coverage $0$ and expected acceptance $2$. Cap~1 has coverage $3/4$, expected acceptance $1$ and shortfall $1$, sampled-leaf stopping has coverage $3/4$, expected acceptance $5/4$ and shortfall $3/4$, and cap~2 has coverage $1/4$, expected acceptance $7/4$ and shortfall $1/4$.

Cap~1 and sampled-leaf stopping have equal coverage but defer different path lengths. For cap~1, covered continuation addresses are $aa$, $ab$ and $ba$, each with probability $1/4$. The first leaves two accepted tokens unused and the others one. Sampled-leaf stopping instead ends at $aaa,aab,ab,ba$ with probabilities $1/8,1/8,1/4,1/4$, each leaving one token unused.

\paragraph{Scored prefixes beyond the trie.}
Let the candidate trie be $T=\{\epsilon,a\}$ and the scored set be $Q=\{\epsilon,a,b\}$ under the same target. Exhaustive traversal has expected acceptance $1/2$, coverage $1/2$ and zero shortfall relative to its own trie. Scoring only the two rows in $T$ gives zero coverage. Making $b$ a candidate raises expected acceptance to $1$, and the added token is accepted exactly in the covered rounds.

\begin{table*}[t]
\centering\small
\begin{tabular*}{\textwidth}{@{\extracolsep{\fill}}ll@{}}
\toprule
Checkpoint & Role \\
\midrule
\href{https://huggingface.co/meta-llama/Llama-3.1-8B-Instruct/tree/0e9e39f249a16976918f6564b8830bc894c89659}{\path{meta-llama/Llama-3.1-8B-Instruct}} & Target \\
\href{https://huggingface.co/yuhuili/EAGLE3-LLaMA3.1-Instruct-8B/tree/ada412b672e293d682423de84a095447bf38a637}{\path{yuhuili/EAGLE3-LLaMA3.1-Instruct-8B}} & EAGLE-3 drafter \\
\href{https://huggingface.co/meta-llama/Llama-3.2-1B-Instruct/tree/9213176726f574b556790deb65791e0c5aa438b6}{\path{meta-llama/Llama-3.2-1B-Instruct}} & Standalone drafter \\
\midrule
\href{https://huggingface.co/Qwen/Qwen3-4B/tree/1cfa9a7208912126459214e8b04321603b3df60c}{\path{Qwen/Qwen3-4B}} & Target \\
\href{https://huggingface.co/AngelSlim/Qwen3-4B_eagle3/tree/fd331e59626c8e95c392381a16ee59d518727fbb}{\path{AngelSlim/Qwen3-4B_eagle3}} & EAGLE-3 drafter \\
\midrule
\href{https://huggingface.co/Qwen/Qwen3-8B/tree/b968826d9c46dd6066d109eabc6255188de91218}{\path{Qwen/Qwen3-8B}} & Target \\
\href{https://huggingface.co/AngelSlim/Qwen3-8B_eagle3/tree/9629dfce7a4a10564dd48d3e5485c3976095653c}{\path{AngelSlim/Qwen3-8B_eagle3}} & EAGLE-3 drafter \\
\href{https://huggingface.co/Qwen/Qwen3-0.6B/tree/c1899de289a04d12100db370d81485cdf75e47ca}{\path{Qwen/Qwen3-0.6B}} & Standalone drafter \\
\midrule
\href{https://huggingface.co/Qwen/Qwen2.5-7B-Instruct/tree/a09a35458c702b33eeacc393d103063234e8bc28}{\path{Qwen/Qwen2.5-7B-Instruct}} & Target \\
\href{https://huggingface.co/ruipeterpan/Qwen2.5-7B-Instruct_EAGLE3_UltraChat/tree/a4de60731c3cdb9a46b2efd27ef3e47c5ec65338}{\path{ruipeterpan/Qwen2.5-7B-Instruct_EAGLE3_UltraChat}} & EAGLE-3 drafter \\
\midrule
\href{https://huggingface.co/Qwen/Qwen2.5-14B-Instruct/tree/cf98f3b3bbb457ad9e2bb7baf9a0125b6b88caa8}{\path{Qwen/Qwen2.5-14B-Instruct}} & Target \\
\href{https://huggingface.co/ruipeterpan/Qwen2.5-14B-Instruct_EAGLE3_UltraChat/tree/93c21df73b040d27398febafbfc6320379bc5fd4}{\path{ruipeterpan/Qwen2.5-14B-Instruct_EAGLE3_UltraChat}} & EAGLE-3 drafter \\
\midrule
\href{https://huggingface.co/Qwen/Qwen2.5-32B-Instruct/tree/5ede1c97bbab6ce5cda5812749b4c0bdf79b18dd}{\path{Qwen/Qwen2.5-32B-Instruct}} & Target \\
\href{https://huggingface.co/ruipeterpan/Qwen2.5-32B-Instruct_EAGLE3_UltraChat/tree/4a161cd6e29856f7a802d56939815ecdb447700f}{\path{ruipeterpan/Qwen2.5-32B-Instruct_EAGLE3_UltraChat}} & EAGLE-3 drafter \\
\bottomrule
\end{tabular*}
\caption{Target checkpoints and their paired EAGLE-3 and standalone drafters. Checkpoint names link to the evaluated Hugging Face revisions.}
\label{tab:models}
\end{table*}
\begin{table*}[t]
\centering\small
\begin{tabular*}{\textwidth}{@{\extracolsep{\fill}}lrrrrrrrr@{}}
\toprule
Target & MT-bench & Translation & Summarization & QA & Math & RAG & HumanEval & All \\
\midrule
\rowcolor{tablegroup}\multicolumn{9}{c}{Greedy decoding} \\
Llama-3.1-8B & 5.53 & 2.65 & 5.10 & 4.52 & 6.05 & 5.48 & 6.68 & 5.51 \\
Qwen3-4B & 3.26 & 2.59 & 3.05 & 3.03 & 3.74 & 3.38 & 3.51 & 3.34 \\
Qwen3-8B & 3.05 & 2.20 & 2.91 & 2.69 & 3.63 & 3.22 & 3.59 & 3.22 \\
Qwen2.5-7B & 4.29 & 3.20 & 3.62 & 3.69 & 5.14 & 4.03 & 4.73 & 4.29 \\
Qwen2.5-14B & 4.48 & 3.24 & 3.81 & 3.71 & 5.26 & 4.13 & 5.00 & 4.49 \\
Qwen2.5-32B & 4.61 & 3.24 & 3.65 & 3.84 & 5.22 & 4.04 & 5.05 & 4.49 \\
\rowcolor{tablegroup}\multicolumn{9}{c}{Sampling at temperature $1$} \\
Llama-3.1-8B & 4.45 & 2.29 & 4.42 & 2.91 & 5.05 & 4.26 & 5.99 & 4.53 \\
Qwen3-4B & 3.21 & 2.57 & 3.02 & 2.85 & 3.63 & 3.28 & 3.45 & 3.25 \\
Qwen3-8B & 2.91 & 2.19 & 2.82 & 2.52 & 3.53 & 3.03 & 3.46 & 3.08 \\
Qwen2.5-7B & 3.68 & 2.86 & 3.03 & 3.05 & 4.53 & 3.48 & 4.38 & 3.75 \\
Qwen2.5-14B & 3.79 & 3.07 & 3.14 & 3.16 & 4.79 & 3.52 & 4.49 & 3.88 \\
Qwen2.5-32B & 3.82 & 3.16 & 3.23 & 3.34 & 4.95 & 3.56 & 4.36 & 3.93 \\
\bottomrule
\end{tabular*}
\caption{Mean tokens emitted in each round, counting accepted draft tokens and the continuation, by task.}
\label{tab:tau}
\end{table*}

\section{Measurement Protocol}
\label{app:setup}
\paragraph{Checkpoints and prompts.}
Table~\ref{tab:models} lists the targets and drafters. Apart from the pipelined runs of Appendix~\ref{app:pipe}, all runs use the EAGLE-3 reference implementation with PyTorch 2.7.1 \citep{paszke2019pytorch} and bfloat16 weights. These runs use eager attention, except that the standalone drafters use scaled dot-product attention as described in Appendix~\ref{app:lm}. Each non-pipelined run uses one NVIDIA A100 80GB, H100 NVL or RTX A6000 GPU, except that the Qwen2.5-32B runs of root and root-draw selection in Tables~\ref{tab:online-main} and~\ref{tab:online}, and their second runs at temperature $1$, use two RTX A6000 GPUs.

Prompts use each target's chat template with one user turn and no added system message, so the Qwen2.5 and Llama-3.1 templates insert their default system text, and the Qwen3 templates use their non-thinking mode.

Spec-Bench supplies the first turns of $80$ MT-bench conversations \citep{zheng2023judging}, ten from each of its eight categories, and $80$ prompts each for translation \citep{bojar2014findings}, summarization \citep{hermann2015teaching,nallapati2016abstractive}, question answering \citep{kwiatkowski2019natural}, mathematical reasoning \citep{cobbe2021training} and retrieval-augmented generation \citep{kwiatkowski2019natural,karpukhin2020dense}. HumanEval supplies $164$ function signatures with docstrings, each used verbatim as the user turn.

Generation stops after the round in which an end-of-sequence token appears or the output reaches the default limit, so an output holds at most $264$ tokens. The twelve default-length runs contain $365{,}922$ verification rounds. The runs of Appendix~\ref{app:long} use a limit of $1{,}024$ new tokens, so their outputs hold at most $1{,}032$ tokens.

\paragraph{Draft pool and tree.}
The drafter first proposes its ten most probable tokens and expands all of them. At each of seven further steps, it adds the ten most probable children of every expanded node to the pool, and after each of the first six it expands the ten new candidates with the highest cumulative draft log probability. The pool thus holds $10$ candidates at depth one and $100$ at each depth from two to eight, $710$ in total, and $70$ of them are expanded.

Candidates are ranked by cumulative draft log probability, then by depth and then by pool index. The verified tree holds the $59$ highest-ranked candidates, and the parent of every candidate is the root or an expanded node. Selection rules that keep draft factors take them from the same cumulative scores. Every rule's score is non-increasing along a path, so with ties broken by depth its $B$ highest-ranked candidates form an ancestor-closed set of exactly $B$ nodes, and the oracle's set maximizes $V$ among all ancestor-closed sets of $B$ candidates.

\paragraph{Verification.}
Under greedy decoding, traversal accepts the child labeled by the target argmax until the argmax labels no child, and that token becomes the continuation. Under sampling, each reached node draws a token from its full target row, without top-k or top-p truncation, with a generator seeded for each prompt and run, and proceeds in the same way.

EAGLE-3's sampling verifier visits the distinct children of a node in order and accepts each with its probability under the current row, otherwise removing it from the row and renormalizing. The probability of accepting the $k$-th visited child is thus the product of the earlier rejection probabilities and its renormalized probability, and this product telescopes to its target probability. If every child is rejected, the continuation is drawn from the row restricted to the remaining tokens. This is the law of exhaustive target sampling.

The acceptance test uses no draft probability, so this law holds given the realized tree. The accepted tokens also form a root path of that tree, and exactly one continuation is emitted, so the verifier has the three properties that Theorem~\ref{thm:characterization} requires of it. The key-value cache then keeps the root and the accepted path, and the continuation becomes the next root.

\paragraph{Scored rows.}
The verification pass scores the root, the verified tree, all expanded nodes and the $160$ highest-ranked candidates under a tree attention mask, about $170$ prefixes in each round. Each position attends only to its own prefix, so the additional prefixes leave the rows of the verified tree unchanged up to rounding.

EAGLE-3's decoder in Appendix~\ref{app:reuse} runs online at its own contexts with a pass that scores only its tree. At $B=59$, the expected acceptance of EAGLE-3's tree on the recorded rounds lies inside the $95\%$ prompt-bootstrap interval of this decoder's realized acceptance in every run. The relative difference is at most $1.1\%$.

Every candidate's parent is scored, so the pass gives the target probability of all $710$ candidate edges. Scored sets of other sizes are defined by address membership, since the row at a covered address is the next round's root row.

\paragraph{Retained rows.}
After each round, we retain the row at the emitted address when the verification pass scored it, together with up to $24$ rows, taken in pool order, at other scored addresses whose final token equals the continuation. When the root token equals the continuation, the root row joins the second group without counting toward the $24$. Each retained row is compared with the next round's root row, computed at the complete emitted context. Total variation and argmax agreement use full-vocabulary softmax probabilities in float32.

\begin{figure*}[t]
\centering
\includegraphics[width=\textwidth]{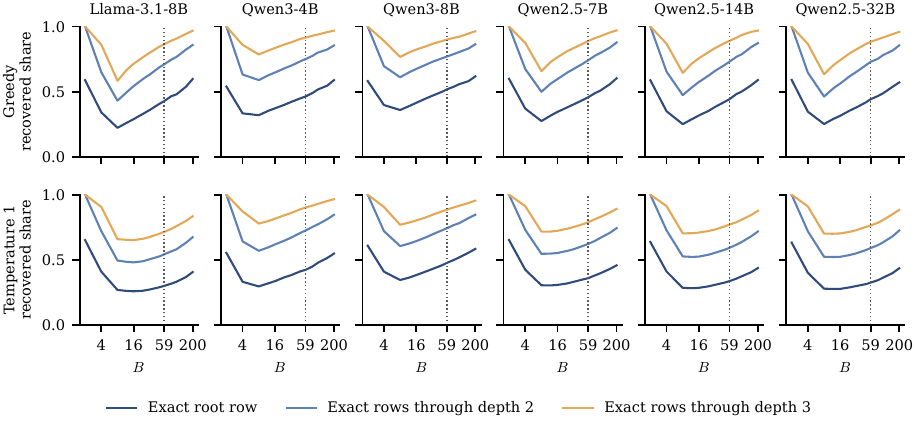}
\caption{Share of the oracle advantage over draft selection recovered by exact rows at increasing depths, for each target and temperature. The dotted line marks EAGLE-3's $59$-node tree.}
\label{fig:recovery-grid}
\end{figure*}

\paragraph{Rows at other addresses.}
A second pass replays each recorded generation through its target. Each round scores the recorded prefixes under the same tree mask and then commits the recorded accepted path and continuation, so contexts, pools and outputs are those of the recorded run. The replayed root rows equal the recorded ones in the nine runs replayed on the capture's GPU type. In the other three, which ran on another GPU type, the probability of a first-level candidate differs from the recorded one by at most $0.007$ on average in each run.

Before the first round, the token-keyed cache holds, for every prompt token, the row computed at its latest occurrence. After each round, it stores the row of every verified prefix under the prefix's final token, keeping the highest-ranked prefix when several end in the same token, and then stores the root's row under the root token.

The same-token row is the previous root's row when the previous root ends in the continuation token, and otherwise the row of the highest-ranked prefix of the previous verified tree that ends in it. The preceding row is the row at the last accepted node, or at the root when no draft token is accepted. Its probability of the continuation is set to zero before its first-level candidates are ranked and before it is compared with the next root row. At every budget, including $B=8$, the three rows come from this replay, whose verified trees are EAGLE-3's $59$-node trees.

\paragraph{Compute.}
The twelve runs take about $9$ GPU hours, including model loading, and the replay takes about $8$ more. Online decoding in Appendix~\ref{app:online} and the covering comparison in Appendix~\ref{app:reuse} together take about $121$ GPU hours, including the PEARL-style and covering decoders, the second runs at temperature $1$ and the further runs of root-draw selection. The $1{,}024$-token runs take about $4$ more GPU hours. The standalone-drafter runs in Appendix~\ref{app:lm} take about $7$ GPU hours, and their pipelined decoding in Appendix~\ref{app:pipe} takes about $50$, counting both GPUs.

\begin{table*}[t]
\centering\footnotesize
\setlength{\tabcolsep}{2pt}
\begin{tabular*}{\textwidth}{@{\extracolsep{\fill}}lrrrrrrr@{}}
\toprule
& & Expected & \multicolumn{3}{c}{Gain over draft selection} & \multicolumn{2}{c}{Recovered oracle advantage, \%} \\
\cmidrule(lr){3-3}\cmidrule(lr){4-6}\cmidrule(lr){7-8}
Target & $B$ & Depth 2 & \cellcolor{tableroot}Root & Root draw & Oracle & \cellcolor{tableroot}Root & Depth 2 \\
\midrule
\rowcolor{tablegroup}\multicolumn{8}{c}{Greedy decoding} \\
Llama-3.1-8B & 8 & 3.925 & \cellcolor{tableroot}0.382 [0.364, 0.401] & 0.468 [0.450, 0.485] & 1.713 [1.657, 1.766] & \cellcolor{tableroot}22 [21, 23] & 43 [42, 44] \\
 & 59 & 4.786 & \cellcolor{tableroot}0.167 [0.157, 0.176] & 0.252 [0.241, 0.265] & 0.390 [0.373, 0.406] & \cellcolor{tableroot}43 [42, 44] & 71 [70, 72] \\
Qwen3-4B & 8 & 2.287 & \cellcolor{tableroot}0.338 [0.328, 0.347] & 0.438 [0.428, 0.449] & 1.056 [1.037, 1.076] & \cellcolor{tableroot}32 [31, 33] & 59 [58, 60] \\
 & 59 & 2.625 & \cellcolor{tableroot}0.176 [0.171, 0.182] & 0.276 [0.270, 0.284] & 0.380 [0.370, 0.389] & \cellcolor{tableroot}46 [45, 48] & 75 [74, 76] \\
Qwen3-8B & 8 & 2.172 & \cellcolor{tableroot}0.363 [0.354, 0.372] & 0.527 [0.517, 0.537] & 1.009 [0.985, 1.032] & \cellcolor{tableroot}36 [35, 37] & 61 [60, 62] \\
 & 59 & 2.485 & \cellcolor{tableroot}0.178 [0.173, 0.183] & 0.342 [0.335, 0.348] & 0.343 [0.333, 0.353] & \cellcolor{tableroot}52 [50, 53] & 77 [76, 78] \\
Qwen2.5-7B & 8 & 2.997 & \cellcolor{tableroot}0.351 [0.338, 0.364] & 0.506 [0.492, 0.519] & 1.277 [1.243, 1.311] & \cellcolor{tableroot}28 [27, 28] & 50 [49, 51] \\
 & 59 & 3.547 & \cellcolor{tableroot}0.157 [0.150, 0.164] & 0.311 [0.303, 0.320] & 0.343 [0.332, 0.354] & \cellcolor{tableroot}46 [45, 47] & 74 [73, 75] \\
Qwen2.5-14B & 8 & 3.104 & \cellcolor{tableroot}0.355 [0.341, 0.367] & 0.492 [0.479, 0.505] & 1.411 [1.376, 1.447] & \cellcolor{tableroot}25 [24, 26] & 47 [46, 48] \\
 & 59 & 3.749 & \cellcolor{tableroot}0.159 [0.153, 0.166] & 0.297 [0.289, 0.304] & 0.360 [0.349, 0.371] & \cellcolor{tableroot}44 [43, 45] & 73 [72, 74] \\
Qwen2.5-32B & 8 & 3.103 & \cellcolor{tableroot}0.348 [0.335, 0.362] & 0.483 [0.469, 0.498] & 1.379 [1.348, 1.413] & \cellcolor{tableroot}25 [24, 26] & 46 [45, 47] \\
 & 59 & 3.748 & \cellcolor{tableroot}0.157 [0.150, 0.164] & 0.292 [0.284, 0.300] & 0.354 [0.344, 0.365] & \cellcolor{tableroot}44 [43, 46] & 73 [72, 74] \\
\rowcolor{tablegroup}\multicolumn{8}{c}{Sampling at temperature $1$} \\
Llama-3.1-8B & 8 & 2.831 & \cellcolor{tableroot}0.210 [0.199, 0.221] & 0.570 [0.553, 0.587] & 0.781 [0.744, 0.816] & \cellcolor{tableroot}27 [26, 28] & 50 [48, 51] \\
 & 59 & 3.720 & \cellcolor{tableroot}0.107 [0.101, 0.112] & 0.342 [0.323, 0.364] & 0.355 [0.338, 0.371] & \cellcolor{tableroot}30 [29, 31] & 54 [53, 55] \\
Qwen3-4B & 8 & 2.167 & \cellcolor{tableroot}0.295 [0.287, 0.302] & 0.451 [0.441, 0.462] & 0.994 [0.977, 1.013] & \cellcolor{tableroot}30 [29, 30] & 57 [56, 58] \\
 & 59 & 2.524 & \cellcolor{tableroot}0.158 [0.154, 0.163] & 0.286 [0.279, 0.293] & 0.374 [0.365, 0.384] & \cellcolor{tableroot}42 [41, 43] & 73 [71, 74] \\
Qwen3-8B & 8 & 1.998 & \cellcolor{tableroot}0.313 [0.306, 0.320] & 0.553 [0.544, 0.562] & 0.907 [0.884, 0.930] & \cellcolor{tableroot}35 [34, 35] & 61 [60, 62] \\
 & 59 & 2.330 & \cellcolor{tableroot}0.160 [0.155, 0.164] & 0.365 [0.358, 0.371] & 0.335 [0.325, 0.345] & \cellcolor{tableroot}48 [47, 49] & 74 [73, 75] \\
Qwen2.5-7B & 8 & 2.350 & \cellcolor{tableroot}0.239 [0.230, 0.248] & 0.577 [0.566, 0.588] & 0.782 [0.752, 0.811] & \cellcolor{tableroot}31 [30, 31] & 55 [54, 56] \\
 & 59 & 2.953 & \cellcolor{tableroot}0.120 [0.115, 0.124] & 0.368 [0.357, 0.380] & 0.333 [0.321, 0.344] & \cellcolor{tableroot}36 [35, 37] & 62 [61, 63] \\
Qwen2.5-14B & 8 & 2.412 & \cellcolor{tableroot}0.226 [0.218, 0.234] & 0.554 [0.542, 0.565] & 0.790 [0.763, 0.816] & \cellcolor{tableroot}29 [28, 30] & 53 [51, 54] \\
 & 59 & 3.098 & \cellcolor{tableroot}0.115 [0.111, 0.120] & 0.345 [0.336, 0.355] & 0.343 [0.333, 0.353] & \cellcolor{tableroot}34 [33, 34] & 59 [58, 60] \\
Qwen2.5-32B & 8 & 2.446 & \cellcolor{tableroot}0.217 [0.209, 0.224] & 0.537 [0.527, 0.547] & 0.778 [0.755, 0.802] & \cellcolor{tableroot}28 [27, 29] & 52 [51, 54] \\
 & 59 & 3.124 & \cellcolor{tableroot}0.111 [0.108, 0.115] & 0.336 [0.329, 0.344] & 0.342 [0.333, 0.351] & \cellcolor{tableroot}33 [32, 33] & 59 [58, 60] \\
\bottomrule
\end{tabular*}
\caption{Selection gains in expected accepted draft tokens and recovered shares of the oracle advantage. Depth 2 under Expected gives acceptance with two exact factors. Brackets give $95\%$ prompt-bootstrap intervals.}
\label{tab:selection-ci}
\end{table*}

\begin{table*}[t]
\centering\small
\begin{tabular*}{\textwidth}{@{\extracolsep{\fill}}lrrrr@{}}
\toprule
& \multicolumn{2}{c}{Root gain} & \multicolumn{2}{c}{Depth-two gain} \\
\cmidrule(lr){2-3}\cmidrule(lr){4-5}
Target & \cellcolor{tableroot}Rounds & \cellcolor{tableroot}Equal tasks & Rounds & Equal tasks \\
\midrule
\rowcolor{tablegroup}\multicolumn{5}{c}{Greedy decoding} \\
Llama-3.1-8B & \cellcolor{tableroot}0.167 & \cellcolor{tableroot}0.169 [0.160, 0.179] & 0.277 & 0.279 [0.266, 0.292] \\
Qwen3-4B & \cellcolor{tableroot}0.176 & \cellcolor{tableroot}0.181 [0.176, 0.187] & 0.285 & 0.286 [0.279, 0.293] \\
Qwen3-8B & \cellcolor{tableroot}0.178 & \cellcolor{tableroot}0.169 [0.163, 0.174] & 0.264 & 0.244 [0.238, 0.250] \\
Qwen2.5-7B & \cellcolor{tableroot}0.157 & \cellcolor{tableroot}0.162 [0.155, 0.170] & 0.254 & 0.260 [0.251, 0.269] \\
Qwen2.5-14B & \cellcolor{tableroot}0.159 & \cellcolor{tableroot}0.161 [0.154, 0.168] & 0.263 & 0.264 [0.254, 0.273] \\
Qwen2.5-32B & \cellcolor{tableroot}0.157 & \cellcolor{tableroot}0.164 [0.156, 0.172] & 0.258 & 0.267 [0.258, 0.277] \\
\rowcolor{tablegroup}\multicolumn{5}{c}{Sampling at temperature $1$} \\
Llama-3.1-8B & \cellcolor{tableroot}0.107 & \cellcolor{tableroot}0.106 [0.101, 0.111] & 0.193 & 0.189 [0.182, 0.197] \\
Qwen3-4B & \cellcolor{tableroot}0.158 & \cellcolor{tableroot}0.164 [0.159, 0.169] & 0.271 & 0.272 [0.266, 0.279] \\
Qwen3-8B & \cellcolor{tableroot}0.160 & \cellcolor{tableroot}0.152 [0.147, 0.158] & 0.248 & 0.231 [0.224, 0.237] \\
Qwen2.5-7B & \cellcolor{tableroot}0.120 & \cellcolor{tableroot}0.120 [0.114, 0.126] & 0.207 & 0.206 [0.197, 0.215] \\
Qwen2.5-14B & \cellcolor{tableroot}0.115 & \cellcolor{tableroot}0.117 [0.112, 0.123] & 0.203 & 0.204 [0.196, 0.211] \\
Qwen2.5-32B & \cellcolor{tableroot}0.111 & \cellcolor{tableroot}0.115 [0.110, 0.119] & 0.202 & 0.206 [0.200, 0.212] \\
\bottomrule
\end{tabular*}
\caption{Gains over draft selection at $B=59$ under round weighting and equal task weighting. Brackets give $95\%$ intervals from $1{,}000$ resamples of prompts within tasks.}
\label{tab:equal-task}
\end{table*}

\section{Additional Results}
\label{app:results}
\subsection{Acceptance length}
\label{app:tau}
Table~\ref{tab:tau} gives the mean number of tokens each round emits, its accepted draft tokens and continuation. Rounds emit $3.22$ to $5.51$ tokens under greedy decoding and $3.08$ to $4.53$ under sampling. Counted in emitted tokens, root selection's gain over draft selection is $2.3\%$ to $5.5\%$ at $B=59$, with a median of $3.5\%$, and $6.1\%$ to $14.2\%$ at $B=8$. Root-draw selection's gain over the PEARL-style decoder is $2.9\%$ to $4.9\%$ at $B=59$ and $7.7\%$ to $9.0\%$ at $B=8$. All four gains divide the difference in expected accepted draft tokens by the baseline's expected emitted tokens, continuation included, so each interval also lies above zero.

\begin{table*}[t]
\centering\small
\begin{tabular*}{\textwidth}{@{\extracolsep{\fill}}lrrrrrrrr@{}}
\toprule
& \multicolumn{2}{c}{$M=100$} & \multicolumn{2}{c}{$M=200$} & \multicolumn{2}{c}{$M=400$} & \multicolumn{2}{c}{$M=710$} \\
\cmidrule(lr){2-3}\cmidrule(lr){4-5}\cmidrule(lr){6-7}\cmidrule(lr){8-9}
Target & Advantage & Root, \% & Advantage & Root, \% & Advantage & Root, \% & Advantage & Root, \% \\
\midrule
\rowcolor{tablegroup}\multicolumn{9}{c}{Greedy decoding} \\
Llama-3.1-8B & 0.186 & 58 & 0.321 & 48 & 0.380 & 43 & 0.390 & 43 \\
Qwen3-4B & 0.135 & 63 & 0.279 & 52 & 0.359 & 47 & 0.380 & 46 \\
Qwen3-8B & 0.125 & 67 & 0.247 & 58 & 0.319 & 53 & 0.343 & 52 \\
Qwen2.5-7B & 0.147 & 59 & 0.267 & 51 & 0.327 & 47 & 0.343 & 46 \\
Qwen2.5-14B & 0.164 & 58 & 0.293 & 49 & 0.348 & 45 & 0.360 & 44 \\
Qwen2.5-32B & 0.159 & 58 & 0.284 & 50 & 0.341 & 45 & 0.354 & 44 \\
\rowcolor{tablegroup}\multicolumn{9}{c}{Sampling at temperature $1$} \\
Llama-3.1-8B & 0.163 & 44 & 0.289 & 34 & 0.343 & 31 & 0.355 & 30 \\
Qwen3-4B & 0.134 & 59 & 0.275 & 48 & 0.354 & 43 & 0.374 & 42 \\
Qwen3-8B & 0.122 & 64 & 0.241 & 54 & 0.311 & 49 & 0.335 & 48 \\
Qwen2.5-7B & 0.139 & 49 & 0.257 & 41 & 0.315 & 37 & 0.333 & 36 \\
Qwen2.5-14B & 0.149 & 48 & 0.272 & 38 & 0.329 & 34 & 0.343 & 34 \\
Qwen2.5-32B & 0.147 & 46 & 0.270 & 37 & 0.328 & 33 & 0.342 & 33 \\
\bottomrule
\end{tabular*}
\caption{Oracle advantage over draft selection at $B=59$ when the pool holds the $M$ highest-ranked candidates, and the share that root selection recovers.}
\label{tab:pool}
\end{table*}

\subsection{Selection values with intervals}
\label{app:selci}
Table~\ref{tab:selection-ci} gives the gains of root, root-draw and oracle selection over draft selection and the share of the oracle advantage that root and depth-two selection recover. The root-draw value of a round is one for the drawn token plus the path mass, summed over first-level candidates, of the candidates that each keeps beneath it. When the pool's first level lacks the drawn token, the tree holds that token alone. This happens in $9\%$ to $16\%$ of greedy rounds and with probability $11\%$ to $19\%$ at temperature $1$, and exhaustive traversal of draft selection's tree, which lies in the same pool, then accepts no draft token, while root-draw selection accepts the drawn one.

Each bootstrap draw resamples prompts with all their rounds and keeps every rule and budget together. Every interval lies above zero at both budgets, as does every interval of a greedy root row entered with its softmax probabilities.

Table~\ref{tab:nodes} expresses the root gain in nodes. For each run, it gives the smallest budget at which draft selection reaches root selection's value at a reference budget and the smallest at which root selection reaches draft selection's, searching every budget up to $710$ within the same pool of depths one to eight and repeating the search in every resample. At $B=8$, draft selection needs $11$ to $22$ nodes to reach root selection's value, and root selection reaches draft selection's with $4$ to $7$.

A PEARL-style decoder \citep{liu2024pearl} uses the root row only for the first draw. It takes the same draw as root-draw selection and counts it as one accepted token, then keeps the part of draft selection's tree beneath the drawn token, or the drawn token alone when that tree lacks it. Its value exceeds that of draft selection by the probability that the drawn token is missing from the tree's first level. Draft selection keeps at most $B-1$ descendants of a kept first-level token, and these are its highest-ranked descendants. Root-draw selection keeps the $B-1$ highest-ranked descendants of the drawn token, or all of them when there are fewer, so its value is at least the PEARL-style value in every round. Under the same cap, both decoders keep at most $B$ nodes.

Table~\ref{tab:pearl} shows that root-draw selection improves expected acceptance over the PEARL-style decoder by up to $14.2\%$ at $B=8$ and $7.0\%$ at $B=59$ under the same node cap. All confidence intervals lie above zero. We also find that root-draw selection retains larger trees on average than the PEARL-style decoder: $42$ to $52$ versus $26$ to $38$ nodes at $B=59$, and $6.6$ to $7.4$ versus $4.1$ to $5.8$ at $B=8$.

A variant closer to PEARL, whose drafter commits to one first token, keeps the drafter's most probable first token with its $B-1$ highest-ranked descendants and verifies the drawn token alone when the draw differs. Its value lies below the PEARL-style value in every run, by $0.5\%$ to $1.5\%$ at $B=8$ and $6.7\%$ to $9.1\%$ at $B=59$ with every interval excluding zero, so Table~\ref{tab:pearl} compares root-draw selection with the stronger of the two.

Table~\ref{tab:pearl} also compares root-draw selection with a PEARL-style decoder whose pre-draw budget is raised, separately in each run, until its pruned trees hold as many nodes on average as root-draw selection's. The raised budget mixes the two nearest whole budgets and comes to $81$ to $96$ candidates at $B=59$ and $10.4$ to $14.0$ at $B=8$. The pruned trees then hold more than $B$ nodes in an expected $38\%$ to $53\%$ of rounds at $B=59$ and $40\%$ to $56\%$ at $B=8$, while root-draw selection's trees never do.

After mean tree sizes are matched, root-draw selection still improves expected acceptance by up to $4.8\%$ at $B=8$ and $2.4\%$ at $B=59$. Every confidence interval remains above zero when matching is repeated in each resample. The draw is the drafter's most probable first token in an expected $52\%$ to $76\%$ of rounds. At $B=59$ the pruned trees hold $66$ to $74$ nodes on average after such a draw, more than root-draw selection's $59$, and only $12$ to $15$ after a draw of another first-level token, where root-draw selection keeps $39$ to $49$. At $B=8$ the pruned trees hold $9.2$ to $11.2$ and $2.1$ to $2.4$ nodes after the two kinds of draw, while root-draw selection's trees hold $8$ after either.

\begin{table*}[t]
\centering\small
\begin{tabular*}{\textwidth}{@{\extracolsep{\fill}}lrrrrrr@{}}
\toprule
& \multicolumn{3}{c}{$B=8$} & \multicolumn{3}{c}{$B=59$} \\
\cmidrule(lr){2-4}\cmidrule(lr){5-7}
Target & Gain, \% & Equal size, \% & Share, \% & Gain, \% & Equal size, \% & Share, \% \\
\midrule
\rowcolor{tablegroup}\multicolumn{7}{c}{Greedy decoding} \\
Llama-3.1-8B & 10.0 [9.4, 10.5] & 2.9 [2.6, 3.3] & 71 & 3.5 [3.3, 3.7] & 0.9 [0.8, 1.0] & 64 \\
Qwen3-4B & 12.8 [12.4, 13.3] & 4.5 [4.2, 4.8] & 55 & 6.7 [6.5, 6.9] & 2.4 [2.2, 2.5] & 60 \\
Qwen3-8B & 13.8 [13.4, 14.2] & 4.4 [4.0, 4.7] & 48 & 6.7 [6.5, 6.9] & 2.2 [2.1, 2.4] & 47 \\
Qwen2.5-7B & 10.7 [10.2, 11.2] & 4.1 [3.8, 4.4] & 55 & 4.2 [4.0, 4.4] & 1.3 [1.2, 1.4] & 46 \\
Qwen2.5-14B & 11.1 [10.6, 11.6] & 4.0 [3.7, 4.4] & 59 & 4.2 [4.0, 4.4] & 1.2 [1.1, 1.3] & 51 \\
Qwen2.5-32B & 10.6 [10.1, 11.1] & 3.9 [3.6, 4.2] & 58 & 4.1 [3.9, 4.3] & 1.2 [1.1, 1.4] & 51 \\
\rowcolor{tablegroup}\multicolumn{7}{c}{Sampling at temperature $1$} \\
Llama-3.1-8B & 12.2 [11.7, 12.7] & 3.1 [2.8, 3.4] & 57 & 4.7 [4.5, 5.0] & 1.4 [1.3, 1.5] & 51 \\
Qwen3-4B & 13.1 [12.6, 13.5] & 4.7 [4.4, 4.9] & 53 & 6.8 [6.6, 7.0] & 2.4 [2.2, 2.5] & 57 \\
Qwen3-8B & 14.2 [13.8, 14.6] & 4.8 [4.5, 5.0] & 45 & 7.0 [6.8, 7.2] & 2.3 [2.2, 2.5] & 44 \\
Qwen2.5-7B & 12.6 [12.1, 13.0] & 4.8 [4.6, 5.1] & 48 & 5.2 [5.1, 5.4] & 1.8 [1.7, 1.9] & 42 \\
Qwen2.5-14B & 13.0 [12.5, 13.4] & 4.7 [4.4, 4.9] & 53 & 5.3 [5.2, 5.5] & 1.8 [1.7, 1.9] & 48 \\
Qwen2.5-32B & 12.2 [11.7, 12.6] & 4.5 [4.2, 4.8] & 52 & 5.2 [5.0, 5.3] & 1.8 [1.7, 1.9] & 48 \\
\bottomrule
\end{tabular*}
\caption{Gain of root-draw selection over the PEARL-style decoder. In Equal size, it preselects more candidates until mean tree sizes are equal. Brackets give $95\%$ intervals. Share is the part of root-draw selection's gain over draft selection that Gain makes up.}
\label{tab:pearl}
\end{table*}

\begin{table*}[t]
\centering\small
\begin{tabular*}{\textwidth}{@{\extracolsep{\fill}}lrrrr@{}}
\toprule
& \multicolumn{2}{c}{Reference budget $B=59$} & \multicolumn{2}{c}{Reference budget $B=8$} \\
\cmidrule(lr){2-3}\cmidrule(lr){4-5}
Target & Draft reaches root & Root reaches draft & Draft reaches root & Root reaches draft \\
\midrule
\rowcolor{tablegroup}\multicolumn{5}{c}{Greedy decoding} \\
Llama-3.1-8B & 94 [92, 96] & 37 [36, 37] & 13 [12, 13] & 7 [7, 7] \\
Qwen3-4B & 118 [116, 121] & 26 [25, 26] & 20 [20, 21] & 4 [4, 4] \\
Qwen3-8B & 130 [127, 134] & 24 [23, 25] & 22 [21, 23] & 4 [4, 4] \\
Qwen2.5-7B & 105 [103, 107] & 32 [32, 33] & 15 [15, 16] & 6 [6, 6] \\
Qwen2.5-14B & 99 [97, 101] & 34 [34, 35] & 14 [14, 15] & 6 [6, 6] \\
Qwen2.5-32B & 100 [98, 102] & 34 [33, 35] & 14 [14, 15] & 6 [6, 6] \\
\rowcolor{tablegroup}\multicolumn{5}{c}{Sampling at temperature $1$} \\
Llama-3.1-8B & 82 [81, 83] & 43 [42, 43] & 11 [11, 11] & 7 [7, 7] \\
Qwen3-4B & 111 [109, 113] & 27 [27, 28] & 18 [18, 19] & 4 [4, 4] \\
Qwen3-8B & 121 [119, 124] & 26 [25, 27] & 20 [20, 21] & 4 [4, 4] \\
Qwen2.5-7B & 92 [91, 94] & 37 [37, 38] & 13 [13, 14] & 6 [6, 6] \\
Qwen2.5-14B & 88 [87, 89] & 40 [39, 40] & 13 [12, 13] & 6 [6, 6] \\
Qwen2.5-32B & 87 [86, 88] & 40 [40, 41] & 12 [12, 13] & 6 [6, 6] \\
\bottomrule
\end{tabular*}
\caption{Smallest budget at which draft selection reaches root selection's expected accepted draft tokens at the reference budget, and at which root selection reaches draft selection's. Brackets give $95\%$ prompt-bootstrap intervals.}
\label{tab:nodes}
\end{table*}

\begin{table*}[t]
\centering\small
\begin{tabular*}{\textwidth}{@{\extracolsep{\fill}}lrrrrrrr@{}}
\toprule
& & \multicolumn{3}{c}{Root selection} & \multicolumn{3}{c}{Root-draw selection} \\
\cmidrule(lr){3-5}\cmidrule(lr){6-8}
Target & EAGLE-3 & Accepted & Gain, \% & Offline, \% & Accepted & Gain, \% & Offline, \% \\
\midrule
\rowcolor{tablegroup}\multicolumn{8}{c}{Greedy decoding} \\
Llama-3.1-8B & 4.509 & 4.721 & 4.7 [3.6, 5.9] & 3.7 & 4.842 & 7.4 [6.3, 8.4] & 5.6 \\
Qwen3-4B & 2.340 & 2.507 & 7.1 [6.5, 7.8] & 7.5 & 2.630 & 12.4 [11.6, 13.2] & 11.8 \\
Qwen3-8B & 2.221 & 2.402 & 8.2 [7.4, 8.9] & 8.0 & 2.606 & 17.4 [16.4, 18.3] & 15.4 \\
Qwen2.5-7B & 3.293 & 3.463 & 5.2 [4.2, 6.1] & 4.8 & 3.683 & 11.9 [10.7, 13.0] & 9.5 \\
Qwen2.5-14B & 3.486 & 3.659 & 5.0 [4.2, 5.8] & 4.6 & 3.815 & 9.4 [8.6, 10.3] & 8.5 \\
Qwen2.5-32B & 3.489 & 3.635 & 4.2 [3.5, 4.9] & 4.5 & 3.811 & 9.2 [8.4, 10.1] & 8.4 \\
\rowcolor{tablegroup}\multicolumn{8}{c}{Sampling at temperature $1$} \\
Llama-3.1-8B & 3.534 & 3.606 & 2.1 [$-$2.4, 7.1] & 3.0 & 3.964 & 12.2 [8.0, 17.0] & 9.7 \\
Qwen3-4B & 2.255 & 2.413 & 7.0 [6.0, 8.0] & 7.0 & 2.574 & 14.1 [13.2, 15.1] & 12.7 \\
Qwen3-8B & 2.082 & 2.257 & 8.4 [7.1, 9.7] & 7.7 & 2.508 & 20.4 [19.0, 21.8] & 17.5 \\
Qwen2.5-7B & 2.749 & 2.886 & 5.0 [2.0, 7.9] & 4.4 & 3.235 & 17.7 [15.2, 20.4] & 13.4 \\
Qwen2.5-14B & 2.883 & 2.983 & 3.5 [1.7, 5.4] & 4.0 & 3.288 & 14.0 [12.1, 16.1] & 11.9 \\
Qwen2.5-32B & 2.929 & 3.037 & 3.7 [1.6, 5.9] & 3.8 & 3.329 & 13.7 [12.0, 15.3] & 11.5 \\
\bottomrule
\end{tabular*}
\caption{Online decoding at $B=59$. Accepted is the realized mean acceptance, and EAGLE-3 that of the recorded runs. Gain over EAGLE-3 has $95\%$ prompt-bootstrap intervals in brackets, and Offline is the gain over draft selection at the recorded contexts.}
\label{tab:online}
\end{table*}

\begin{table*}[t]
\centering\small
\begin{tabular*}{\textwidth}{@{\extracolsep{\fill}}lrrrrrrrr@{}}
\toprule
& \multicolumn{2}{c}{First tokens differ, \%} & \multicolumn{2}{c}{Oracle advantage} & \multicolumn{2}{c}{Root recovery, \%} & \multicolumn{2}{c}{Depth-two recovery, \%} \\
\cmidrule(lr){2-3}\cmidrule(lr){4-5}\cmidrule(lr){6-7}\cmidrule(lr){8-9}
Target & Rounds & Root gain & Agree & Differ & \cellcolor{tableroot}Agree & \cellcolor{tableroot}Differ & Agree & Differ \\
\midrule
\rowcolor{tablegroup}\multicolumn{9}{c}{Greedy decoding} \\
Llama-3.1-8B & 24 & 60 & 0.371 & 0.448 & \cellcolor{tableroot}24 & \cellcolor{tableroot}92 & 61 & 96 \\
Qwen3-4B & 36 & 75 & 0.371 & 0.395 & \cellcolor{tableroot}19 & \cellcolor{tableroot}93 & 62 & 97 \\
Qwen3-8B & 45 & 78 & 0.352 & 0.333 & \cellcolor{tableroot}20 & \cellcolor{tableroot}94 & 61 & 97 \\
Qwen2.5-7B & 34 & 69 & 0.340 & 0.349 & \cellcolor{tableroot}22 & \cellcolor{tableroot}92 & 62 & 96 \\
Qwen2.5-14B & 31 & 67 & 0.358 & 0.364 & \cellcolor{tableroot}22 & \cellcolor{tableroot}93 & 62 & 97 \\
Qwen2.5-32B & 31 & 68 & 0.345 & 0.375 & \cellcolor{tableroot}21 & \cellcolor{tableroot}93 & 61 & 97 \\
\rowcolor{tablegroup}\multicolumn{9}{c}{Sampling at temperature $1$} \\
Llama-3.1-8B & 31 & 57 & 0.371 & 0.320 & \cellcolor{tableroot}18 & \cellcolor{tableroot}62 & 45 & 77 \\
Qwen3-4B & 37 & 75 & 0.375 & 0.374 & \cellcolor{tableroot}17 & \cellcolor{tableroot}86 & 60 & 94 \\
Qwen3-8B & 46 & 78 & 0.351 & 0.316 & \cellcolor{tableroot}19 & \cellcolor{tableroot}85 & 59 & 94 \\
Qwen2.5-7B & 37 & 68 & 0.348 & 0.306 & \cellcolor{tableroot}17 & \cellcolor{tableroot}72 & 51 & 85 \\
Qwen2.5-14B & 34 & 64 & 0.354 & 0.321 & \cellcolor{tableroot}18 & \cellcolor{tableroot}68 & 48 & 82 \\
Qwen2.5-32B & 33 & 63 & 0.354 & 0.318 & \cellcolor{tableroot}17 & \cellcolor{tableroot}67 & 49 & 82 \\
\bottomrule
\end{tabular*}
\caption{Selection at $B=59$ split by whether the drafter's most probable first token is the target's most probable token. Oracle advantage is in expected accepted draft tokens, and recovery is the share of it that each rule attains.}
\label{tab:agreement}
\end{table*}

\subsection{Online decoding}
\label{app:online}
The online decoders of Section~\ref{sec:selection} that take the root row from a separate target pass run on all six targets, with the checkpoints, prompts, pools, output limit and seeding of Appendix~\ref{app:setup}. In each round, the drafter grows its pool from the committed context, and a first target pass scores the root alone to give the root row. The pool does not depend on this row, so running the first pass while the drafter works would leave every acceptance value unchanged, as the pipelined runs of Appendix~\ref{app:pipe} confirm for the standalone drafters.

The rule selects the tree with that row, and a second pass scores the root and the tree under the tree mask. Exhaustive target sampling then traverses the tree, taking its first draw from the root row. Root-draw selection makes this draw before it selects the tree, and the PEARL-style decoder makes it before it cuts draft selection's $59$-node tree to the part beneath the drawn token, or to the drawn token alone when that tree lacks it. The key-value cache keeps the root and the accepted path, and the continuation becomes the next root.

The two passes compute the root row in batches of different sizes, and the two computations share their argmax in $98.6\%$ to $99.6\%$ of rounds and, in the PEARL-style runs, differ by a mean total variation of $0.0041$ to $0.0129$. For Qwen3-4B and Qwen2.5-7B, whose recorded runs and the runs above all used RTX A6000 GPUs, the matched pairs of Section~\ref{sec:rows} share their argmax more often than these two computations in every decoder at both temperatures, and they differ by a smaller mean total variation than the two computations in the PEARL-style runs.

The two rows of a matched pair have the same token, position and key-value cache, and the replay of Appendix~\ref{app:setup}, run with the recorded batches on the capture's GPU type, reproduces the recorded root rows exactly, so the two rows differ only because different batches compute them.

Table~\ref{tab:online} gives the realized acceptance of root and root-draw selection, counting the drawn token under root-draw selection, and Table~\ref{tab:online-pearl} that of root-draw selection and the PEARL-style decoder behind Table~\ref{tab:online-main}. Realized acceptance divides the accepted tokens of all rounds by the number of rounds. Each decoder processes one prompt at a time, with a generator seeded for each prompt, run and decoder. An online run and the recorded run follow different outputs after their first differing token, so each resample draws prompts jointly for both.

Under greedy decoding, the two outputs agree up to a first token where bfloat16 rounding changes an argmax, and $38\%$ to $59\%$ of the emitted tokens precede it. On this shared text, a round of either decoder that starts where a recorded round starts has the recorded context, and its acceptance equals its rule's value on the recorded rows in $97.2\%$ to $98.5\%$ of such rounds. On the same text, root-draw decoders start $6.8\%$ to $12.9\%$ of their rounds where \mbox{EAGLE-3} accepted no draft token, against $9.9\%$ to $17.8\%$ of \mbox{EAGLE-3}'s rounds, a relative reduction of $25\%$ to $31\%$.

At temperature $1$ the outputs diverge within the first rounds, and the sampling noise of two separate outputs enters the comparison. A second run of each root decoder at temperature $1$ therefore scores the union of its tree and EAGLE-3's draft-ranked tree, and the union gives the expected acceptance of both trees at the same contexts. At these contexts, the root gain over draft ranking is $3.0\%$ for Llama-3.1-8B, $7.1\%$ for Qwen3-4B, $7.5\%$ for Qwen3-8B, $4.2\%$ for Qwen2.5-7B, $4.1\%$ for Qwen2.5-14B and $3.9\%$ for Qwen2.5-32B, with every interval above $2.8\%$. Each such run's realized mean acceptance is within $0.01$ tokens of its trees' expected acceptance.

A further run of root-draw selection on every target at both temperatures also scores, in each tree pass, the drawn token's $160$ highest-ranked pool descendants. The verifier still sees only root-draw selection's tree, and the extra prefixes leave its rows unchanged up to rounding. The run uses the seeds of the first root-draw run, but that rounding can make the two outputs differ. The target probabilities of these descendants give, at this run's contexts, the expected acceptance of the PEARL-style tree for every pre-draw budget up to $161$ candidates.

As in Table~\ref{tab:pearl}, the budget is raised until the mean tree sizes are equal, which takes $81$ to $95$ candidates, and the enlarged trees then hold more than $59$ nodes in $39\%$ to $54\%$ of rounds. Table~\ref{tab:online-matched} gives root-draw selection's gain over these trees, $1.0\%$ to $2.4\%$ against $0.9\%$ to $2.4\%$ at the recorded contexts, and with the budget matched again in every resample, each interval lies above zero. Each run's realized mean acceptance is within $0.02$ tokens of its trees' expected acceptance.

Root-draw selection and the PEARL-style decoder also run online with the same prompts and seeding under a cap of $8$ nodes, where the PEARL-style decoder cuts draft selection's $8$-node tree. In Table~\ref{tab:online-pearl-b8}, root-draw selection's realized gain is $10.7\%$ to $13.7\%$, with every interval above $9.8\%$, and each run's gain is within $0.8$ percentage points of its offline gain.

\begin{table*}[t]
\centering\small
\begin{tabular*}{\textwidth}{@{\extracolsep{\fill}}lrrrrr@{}}
\toprule
& \multicolumn{2}{c}{Accepted} & & & \\
\cmidrule(lr){2-3}
Target & PEARL-style & Root draw & Gain, \% & Offline, \% & Root TV \\
\midrule
\rowcolor{tablegroup}\multicolumn{6}{c}{Greedy decoding} \\
Llama-3.1-8B & 4.686 & 4.842 & 3.3 [2.6, 4.1] & 3.5 & 0.0041 \\
Qwen3-4B & 2.463 & 2.630 & 6.7 [6.2, 7.3] & 6.7 & 0.0057 \\
Qwen3-8B & 2.429 & 2.606 & 7.3 [6.6, 8.0] & 6.7 & 0.0063 \\
Qwen2.5-7B & 3.523 & 3.683 & 4.5 [3.6, 5.5] & 4.2 & 0.0111 \\
Qwen2.5-14B & 3.676 & 3.815 & 3.8 [3.1, 4.4] & 4.2 & 0.0064 \\
Qwen2.5-32B & 3.677 & 3.811 & 3.7 [3.0, 4.3] & 4.1 & 0.0055 \\
\rowcolor{tablegroup}\multicolumn{6}{c}{Sampling at temperature 1} \\
Llama-3.1-8B & 3.807 & 3.964 & 4.1 [1.7, 6.8] & 4.7 & 0.0053 \\
Qwen3-4B & 2.402 & 2.574 & 7.1 [6.4, 7.9] & 6.8 & 0.0058 \\
Qwen3-8B & 2.336 & 2.508 & 7.3 [6.3, 8.3] & 7.0 & 0.0064 \\
Qwen2.5-7B & 3.060 & 3.235 & 5.7 [4.1, 7.4] & 5.2 & 0.0128 \\
Qwen2.5-14B & 3.145 & 3.288 & 4.5 [3.0, 5.9] & 5.3 & 0.0074 \\
Qwen2.5-32B & 3.178 & 3.329 & 4.7 [3.4, 6.1] & 5.2 & 0.0062 \\
\bottomrule
\end{tabular*}
\caption{Online decoding at $B=59$. Accepted is realized mean acceptance. Gain compares root-draw selection with the PEARL-style decoder; brackets give $95\%$ intervals. Offline evaluates shared recorded contexts. Root TV compares the two passes' root rows by mean total variation.}
\label{tab:online-pearl}
\end{table*}

\begin{table*}[t]
\centering\small
\begin{tabular*}{\textwidth}{@{\extracolsep{\fill}}lrrrrr@{}}
\toprule
& \multicolumn{2}{c}{Accepted} & & & \\
\cmidrule(lr){2-3}
Target & PEARL-style & Root draw & Gain, \% & Offline, \% & Root TV \\
\midrule
\rowcolor{tablegroup}\multicolumn{6}{c}{Greedy decoding} \\
Llama-3.1-8B & 3.365 & 3.724 & 10.7 [9.9, 11.4] & 10.0 & 0.0043 \\
Qwen3-4B & 1.881 & 2.122 & 12.8 [12.2, 13.4] & 12.8 & 0.0057 \\
Qwen3-8B & 1.872 & 2.130 & 13.7 [13.1, 14.4] & 13.8 & 0.0063 \\
Qwen2.5-7B & 2.647 & 2.940 & 11.1 [10.1, 12.0] & 10.7 & 0.0113 \\
Qwen2.5-14B & 2.710 & 3.003 & 10.8 [10.1, 11.5] & 11.1 & 0.0068 \\
Qwen2.5-32B & 2.724 & 3.022 & 10.9 [10.3, 11.6] & 10.6 & 0.0055 \\
\rowcolor{tablegroup}\multicolumn{6}{c}{Sampling at temperature 1} \\
Llama-3.1-8B & 2.790 & 3.139 & 12.5 [10.4, 14.7] & 12.2 & 0.0053 \\
Qwen3-4B & 1.849 & 2.088 & 12.9 [12.1, 13.7] & 13.1 & 0.0061 \\
Qwen3-8B & 1.815 & 2.062 & 13.6 [12.7, 14.5] & 14.2 & 0.0063 \\
Qwen2.5-7B & 2.309 & 2.605 & 12.8 [11.2, 14.5] & 12.6 & 0.0129 \\
Qwen2.5-14B & 2.344 & 2.632 & 12.3 [10.8, 13.6] & 13.0 & 0.0074 \\
Qwen2.5-32B & 2.365 & 2.668 & 12.8 [11.6, 14.2] & 12.2 & 0.0063 \\
\bottomrule
\end{tabular*}
\caption{Online decoding at $B=8$. Gain compares root-draw selection with the PEARL-style decoder in realized acceptance; brackets give $95\%$ intervals. Offline evaluates shared recorded contexts. Root TV compares the two passes' root rows by mean total variation.}
\label{tab:online-pearl-b8}
\end{table*}

\begin{table*}[t]
\centering\small
\begin{tabular*}{\textwidth}{@{\extracolsep{\fill}}lrrrr@{}}
\toprule
& \multicolumn{2}{c}{Greedy decoding} & \multicolumn{2}{c}{Sampling at temperature $1$} \\
\cmidrule(lr){2-3}\cmidrule(lr){4-5}
Target & Online, \% & Offline, \% & Online, \% & Offline, \% \\
\midrule
Llama-3.1-8B & 1.0 [0.9, 1.1] & 0.9 [0.8, 1.0] & 1.4 [1.3, 1.5] & 1.4 [1.3, 1.5] \\
Qwen3-4B & 2.3 [2.1, 2.4] & 2.4 [2.2, 2.5] & 2.4 [2.3, 2.5] & 2.4 [2.2, 2.5] \\
Qwen3-8B & 2.1 [2.0, 2.3] & 2.2 [2.1, 2.4] & 2.3 [2.1, 2.4] & 2.3 [2.2, 2.5] \\
Qwen2.5-7B & 1.3 [1.2, 1.4] & 1.3 [1.2, 1.4] & 1.9 [1.8, 2.0] & 1.8 [1.7, 1.9] \\
Qwen2.5-14B & 1.2 [1.1, 1.3] & 1.2 [1.1, 1.3] & 1.7 [1.6, 1.8] & 1.8 [1.7, 1.9] \\
Qwen2.5-32B & 1.2 [1.1, 1.4] & 1.2 [1.1, 1.4] & 1.7 [1.6, 1.8] & 1.8 [1.7, 1.9] \\
\bottomrule
\end{tabular*}
\caption{Gain in expected acceptance of root-draw selection over the PEARL-style tree enlarged to equal mean size at $B=59$, with $95\%$ intervals. Online evaluates both trees at online contexts; Offline uses the shared recorded contexts.}
\label{tab:online-matched}
\end{table*}

\begin{table*}[t]
\centering\small
\setlength{\tabcolsep}{3pt}
\begin{tabular*}{\textwidth}{@{\extracolsep{\fill}}lrrrrrrrrrrrr@{}}
\toprule
& \multicolumn{6}{c}{$B=8$} & \multicolumn{6}{c}{$B=59$} \\
\cmidrule(lr){2-7}\cmidrule(lr){8-13}
& & \multicolumn{2}{c}{Calibrated} & \multicolumn{2}{c}{Reranked} & & & \multicolumn{2}{c}{Calibrated} & \multicolumn{2}{c}{Reranked} & \\
\cmidrule(lr){3-4}\cmidrule(lr){5-6}\cmidrule(lr){9-10}\cmidrule(lr){11-12}
Target & Draft & First & All & Base & Rows & \cellcolor{tableroot}Root & Draft & First & All & Base & Rows & \cellcolor{tableroot}Root \\
\midrule
\rowcolor{tablegroup}\multicolumn{13}{c}{Greedy decoding} \\
Llama-3.1-8B & 3.186 & 3.188 & 3.168 & 3.187 & 3.234 & \cellcolor{tableroot}3.568 & 4.509 & 4.510 & 4.493 & 4.509 & 4.536 & \cellcolor{tableroot}4.676 \\
Qwen3-4B & 1.664 & 1.667 & 1.708 & 1.674 & 1.722 & \cellcolor{tableroot}2.002 & 2.340 & 2.342 & 2.356 & 2.343 & 2.372 & \cellcolor{tableroot}2.516 \\
Qwen3-8B & 1.556 & 1.563 & 1.564 & 1.581 & 1.650 & \cellcolor{tableroot}1.919 & 2.221 & 2.229 & 2.223 & 2.241 & 2.272 & \cellcolor{tableroot}2.399 \\
Qwen2.5-7B & 2.359 & 2.359 & 2.348 & 2.360 & 2.399 & \cellcolor{tableroot}2.710 & 3.293 & 3.293 & 3.273 & 3.294 & 3.311 & \cellcolor{tableroot}3.450 \\
Qwen2.5-14B & 2.435 & 2.434 & 2.415 & 2.436 & 2.475 & \cellcolor{tableroot}2.790 & 3.486 & 3.485 & 3.443 & 3.482 & 3.502 & \cellcolor{tableroot}3.645 \\
Qwen2.5-32B & 2.464 & 2.463 & 2.447 & 2.458 & 2.490 & \cellcolor{tableroot}2.812 & 3.489 & 3.488 & 3.449 & 3.488 & 3.502 & \cellcolor{tableroot}3.646 \\
\rowcolor{tablegroup}\multicolumn{13}{c}{Sampling at temperature $1$} \\
Llama-3.1-8B & 2.444 & 2.447 & 2.437 & 2.451 & 2.482 & \cellcolor{tableroot}2.655 & 3.528 & 3.527 & 3.509 & 3.527 & 3.544 & \cellcolor{tableroot}3.634 \\
Qwen3-4B & 1.600 & 1.606 & 1.648 & 1.613 & 1.656 & \cellcolor{tableroot}1.895 & 2.253 & 2.256 & 2.276 & 2.257 & 2.283 & \cellcolor{tableroot}2.411 \\
Qwen3-8B & 1.449 & 1.459 & 1.462 & 1.475 & 1.537 & \cellcolor{tableroot}1.762 & 2.082 & 2.095 & 2.090 & 2.105 & 2.133 & \cellcolor{tableroot}2.241 \\
Qwen2.5-7B & 1.922 & 1.928 & 1.940 & 1.931 & 1.958 & \cellcolor{tableroot}2.161 & 2.746 & 2.753 & 2.746 & 2.755 & 2.768 & \cellcolor{tableroot}2.866 \\
Qwen2.5-14B & 1.996 & 2.000 & 2.005 & 2.002 & 2.025 & \cellcolor{tableroot}2.221 & 2.895 & 2.899 & 2.877 & 2.901 & 2.914 & \cellcolor{tableroot}3.010 \\
Qwen2.5-32B & 2.039 & 2.044 & 2.054 & 2.045 & 2.066 & \cellcolor{tableroot}2.256 & 2.921 & 2.926 & 2.911 & 2.924 & 2.937 & \cellcolor{tableroot}3.033 \\
\bottomrule
\end{tabular*}
\caption{Expected accepted draft tokens after cross-fitting calibrated or reranked draft factors on complementary prompt halves. Base uses draft and contextual features; Rows adds retained rows from other addresses. Blue cells mark selection with the exact root row.}
\label{tab:calib}
\end{table*}

\begin{table*}[t]
\centering\small
\begin{tabular*}{\textwidth}{@{\extracolsep{\fill}}lrrrrrr@{}}
\toprule
& & & \multicolumn{2}{c}{Late rounds} & \multicolumn{2}{c}{Whole outputs} \\
\cmidrule(lr){4-5}\cmidrule(lr){6-7}
Target & Late, \% & Rounds & $B=59$ & $B=8$ & $B=59$ & $B=8$ \\
\midrule
\rowcolor{tablegroup}\multicolumn{7}{c}{Greedy decoding} \\
Llama-3.1-8B & 41 & 7{,}374 & 2.7 [2.4, 3.1] & 8.5 [7.3, 9.8] & 3.5 [3.3, 3.7] & 11.0 [10.3, 11.7] \\
Qwen3-8B & 48 & 19{,}346 & 7.6 [7.2, 8.1] & 22.6 [20.9, 24.1] & 7.8 [7.6, 8.0] & 23.0 [22.2, 23.9] \\
\rowcolor{tablegroup}\multicolumn{7}{c}{Sampling at temperature $1$} \\
Llama-3.1-8B & 41 & 12{,}334 & 2.8 [2.6, 3.1] & 7.6 [6.9, 8.4] & 3.0 [2.9, 3.1] & 8.4 [8.0, 8.9] \\
Qwen3-8B & 48 & 21{,}850 & 7.3 [7.0, 7.6] & 21.4 [20.4, 22.5] & 7.4 [7.3, 7.6] & 21.5 [20.9, 22.2] \\
\bottomrule
\end{tabular*}
\caption{Root selection's gains over draft ranking under a $1{,}024$-token limit, in percent, with $95\%$ prompt-bootstrap intervals in brackets. Late is the share of outputs with rounds starting once the default output limit is reached; Rounds counts these rounds.}
\label{tab:long}
\end{table*}

\subsection{Tasks, weighting and targets}
\label{app:taskweights}
Round averages weight long generations more heavily. Averaging the rounds of each prompt first, so that every prompt weighs the same, gives relative root gains of $2.9\%$ to $7.9\%$ at $B=59$ and $8.6\%$ to $23.4\%$ at $B=8$, with every interval above zero. Table~\ref{tab:equal-task} instead averages the seven tasks equally, with intervals from $1{,}000$ resamples of prompts within each task. With prompts resampled within each task, every task interval of the root and root-draw gains lies above zero at both budgets in every run.

Figure~\ref{fig:recovery-grid} shows the recovered share of the oracle advantage for each target at the budgets evaluated in Figure~\ref{fig:depth}, $B=2$, $4$, $8$, $12$, $16$, $24$, $32$, $48$, $59$, $80$, $100$, $150$ and $200$.

\subsection{Longer outputs}
\label{app:long}
Llama-3.1-8B and Qwen3-8B also decode all $644$ prompts under a $1{,}024$-token limit at both temperatures, on RTX A6000 GPUs and with the code and settings of Appendix~\ref{app:setup} otherwise unchanged. A late round starts with an output length at or above the default limit, so the default-length runs contain none. Late rounds occur in $41\%$ to $48\%$ of the outputs, and HumanEval prompts account for about half of these outputs. At most $2.5\%$ of all outputs reach the limit.

Table~\ref{tab:long} gives the late-round and whole-output gains of root selection over draft ranking. Whole-output gains remain within $0.3$ percentage points of the default-length runs at $B=59$ and within $1.1$ at $B=8$. Root-draw selection also improves late-round expected acceptance over draft ranking at both budgets, with no interval below $3.4\%$.

\subsection{Pool size}
Table~\ref{tab:pool} restricts every rule to the $M$ highest-ranked candidates of the pool and reports the oracle advantage at $B=59$ with the share that root selection recovers. The advantage grows with the pool, reaching $0.12$ to $0.19$ tokens at $M=100$ and $0.33$ to $0.39$ at $M=710$, and root selection recovers $44\%$ to $67\%$ of it at $M=100$.

\begin{table*}[t]
\centering\small
\begin{tabular*}{\textwidth}{@{\extracolsep{\fill}}lrrrrrrr@{}}
\toprule
& & \multicolumn{6}{c}{Coverage, \%, with $E$ scored frontier candidates} \\
\cmidrule(lr){3-8}
Target & Frontier size & $E=10$ & $E=20$ & $E=40$ & $E=80$ & $E=160$ & All \\
\midrule
\rowcolor{tablegroup}\multicolumn{8}{c}{Greedy decoding} \\
Llama-3.1-8B & 407 & 5.9 & 10.4 & 16.9 & 23.7 & 29.4 & 32.8 \\
Qwen3-4B & 388 & 4.4 & 7.9 & 13.7 & 21.4 & 28.5 & 34.5 \\
Qwen3-8B & 381 & 4.0 & 7.1 & 11.9 & 18.2 & 24.6 & 30.4 \\
Qwen2.5-7B & 416 & 4.9 & 8.5 & 13.4 & 19.1 & 24.5 & 28.5 \\
Qwen2.5-14B & 416 & 5.1 & 9.2 & 14.9 & 21.4 & 26.7 & 30.2 \\
Qwen2.5-32B & 417 & 5.1 & 8.9 & 14.3 & 20.6 & 25.9 & 29.7 \\
\rowcolor{tablegroup}\multicolumn{8}{c}{Sampling at temperature $1$} \\
Llama-3.1-8B & 399 & 5.2 & 9.3 & 15.3 & 22.2 & 28.1 & 32.1 \\
Qwen3-4B & 387 & 4.3 & 7.8 & 13.4 & 20.8 & 28.1 & 33.9 \\
Qwen3-8B & 380 & 3.8 & 6.9 & 11.6 & 17.6 & 23.8 & 29.5 \\
Qwen2.5-7B & 417 & 4.3 & 7.7 & 12.7 & 18.6 & 23.9 & 28.4 \\
Qwen2.5-14B & 415 & 4.8 & 8.4 & 13.7 & 20.2 & 26.0 & 30.3 \\
Qwen2.5-32B & 418 & 4.8 & 8.3 & 13.6 & 19.9 & 25.7 & 30.1 \\
\bottomrule
\end{tabular*}
\caption{Coverage of exhaustive traversal when the pass also scores the first $E$ frontier candidates in draft order. Frontier size is the mean number of pool candidates outside the tree whose parents lie in it.}
\label{tab:frontier}
\end{table*}

\subsection{First-token agreement}
\label{app:agreement}
Table~\ref{tab:agreement} gives the first-token split of Section~\ref{sec:mechanism} for each run. The first tokens agree when the drafter's most probable first token is the target's most probable token at the current root. In those rounds, rows through depth two recover $45\%$ to $62\%$ of the oracle advantage at $B=59$. In every run, the rounds in which the first tokens differ carry a larger share of the root gain than of the rounds.

The root gain is also larger where the target is uncertain. At $B=59$, rounds in the upper two quintiles of the root row's softmax entropy carry $1.6$ to $12.8$ times the gain of the lower two. These quintiles are taken over the rounds of each run, and under greedy decoding the mean root gain rises from each quintile to the next in every run.

\begin{table*}[t]
\centering\small
\begin{tabular*}{\textwidth}{@{\extracolsep{\fill}}lrrrrr@{}}
\toprule
& & \multicolumn{1}{c}{Tree rows} & \multicolumn{3}{c}{All scored rows} \\
\cmidrule(lr){3-3}\cmidrule(lr){4-6}
Target & Successor rounds & Token, \% & Token, \% & Address, \% & Token only, \% \\
\midrule
\rowcolor{tablegroup}\multicolumn{6}{c}{Greedy decoding} \\
Llama-3.1-8B & 19,340 & 41.7 & 61.1 & 25.1 & 36.0 \\
Qwen3-4B & 35,971 & 27.8 & 47.4 & 22.9 & 24.5 \\
Qwen3-8B & 36,856 & 29.1 & 46.9 & 20.4 & 26.5 \\
Qwen2.5-7B & 27,251 & 31.3 & 46.7 & 20.5 & 26.3 \\
Qwen2.5-14B & 24,556 & 30.9 & 47.9 & 22.8 & 25.2 \\
Qwen2.5-32B & 24,712 & 30.5 & 46.8 & 21.9 & 25.0 \\
\rowcolor{tablegroup}\multicolumn{6}{c}{Sampling at temperature $1$} \\
Llama-3.1-8B & 24,033 & 36.2 & 55.0 & 23.6 & 31.4 \\
Qwen3-4B & 36,639 & 26.9 & 46.5 & 22.6 & 23.9 \\
Qwen3-8B & 38,758 & 28.6 & 45.9 & 19.7 & 26.2 \\
Qwen2.5-7B & 31,693 & 28.4 & 44.1 & 20.1 & 24.0 \\
Qwen2.5-14B & 29,387 & 28.5 & 45.2 & 21.7 & 23.5 \\
Qwen2.5-32B & 28,998 & 28.1 & 44.3 & 20.8 & 23.5 \\
\bottomrule
\end{tabular*}
\caption{Successor rounds whose root's final token or root address appears among the prefixes scored in the previous round. Tree rows are EAGLE-3's $60$ verified prefixes, and all scored rows are the roughly $170$ prefixes of our verification pass.}
\label{tab:addr}
\end{table*}

\begin{table*}[t]
\centering\small
\begin{tabular*}{\textwidth}{@{\extracolsep{\fill}}lrrrrrrr@{}}
\toprule
& \multicolumn{4}{c}{Matched address} & \multicolumn{3}{c}{Matched token only} \\
\cmidrule(lr){2-5}\cmidrule(lr){6-8}
Target & Pairs & Mean TV & 99th pct.\ TV & Argmax agree, \% & Pairs & Mean TV & Argmax agree, \% \\
\midrule
\rowcolor{tablegroup}\multicolumn{8}{c}{Greedy decoding} \\
Llama-3.1-8B & 4,862 & 0.0036 & 0.0316 & 99.7 & 54,630 & 0.416 & 59.1 \\
Qwen3-4B & 8,254 & 0.0043 & 0.0603 & 99.5 & 106,260 & 0.390 & 61.3 \\
Qwen3-8B & 7,505 & 0.0045 & 0.0602 & 99.4 & 96,948 & 0.378 & 62.7 \\
Qwen2.5-7B & 5,575 & 0.0087 & 0.0734 & 99.2 & 62,228 & 0.457 & 54.1 \\
Qwen2.5-14B & 5,592 & 0.0069 & 0.0576 & 99.2 & 55,155 & 0.435 & 56.1 \\
Qwen2.5-32B & 5,410 & 0.0061 & 0.0575 & 99.3 & 56,493 & 0.440 & 55.7 \\
\rowcolor{tablegroup}\multicolumn{8}{c}{Sampling at temperature $1$} \\
Llama-3.1-8B & 5,671 & 0.0043 & 0.0322 & 99.5 & 59,114 & 0.407 & 59.2 \\
Qwen3-4B & 8,290 & 0.0045 & 0.0618 & 99.5 & 104,419 & 0.385 & 61.7 \\
Qwen3-8B & 7,634 & 0.0045 & 0.0591 & 99.5 & 99,252 & 0.388 & 61.5 \\
Qwen2.5-7B & 6,386 & 0.0091 & 0.0763 & 99.1 & 65,157 & 0.470 & 52.7 \\
Qwen2.5-14B & 6,373 & 0.0076 & 0.0598 & 99.3 & 59,350 & 0.444 & 54.9 \\
Qwen2.5-32B & 6,037 & 0.0066 & 0.0555 & 99.4 & 58,872 & 0.431 & 56.5 \\
\bottomrule
\end{tabular*}
\caption{Agreement between retained rows and the next round's root row. Matched rows share the complete emitted prefix, while token-only rows share only its final token. Total variation uses softmax probabilities.}
\label{tab:rowvalidity}
\end{table*}

\begin{table*}[t]
\centering\small
\begin{tabular*}{\textwidth}{@{\extracolsep{\fill}}llrrrrrr@{}}
\toprule
& & \multicolumn{3}{c}{$B=8$} & \multicolumn{3}{c}{$B=59$} \\
\cmidrule(lr){3-5}\cmidrule(lr){6-8}
Target & Drafter & Draft & Root, \% & Root draw, \% & Draft & Root, \% & Root draw, \% \\
\midrule
\rowcolor{tablegroup}\multicolumn{8}{c}{Greedy decoding} \\
Llama-3.1-8B & Llama-3.2-1B & 4.843 & 11.3 [10.6, 12.0] & 9.8 [9.2, 10.4] & 6.706 & 2.7 [2.5, 2.9] & 2.6 [2.4, 2.8] \\
Qwen3-8B & Qwen3-0.6B & 3.556 & 16.8 [15.8, 17.8] & 13.2 [12.5, 14.0] & 5.234 & 5.0 [4.6, 5.3] & 4.6 [4.3, 4.9] \\
\rowcolor{tablegroup}\multicolumn{8}{c}{Sampling at temperature 1} \\
Llama-3.1-8B & Llama-3.2-1B & 3.391 & 6.2 [5.8, 6.7] & 13.4 [12.8, 14.0] & 5.045 & 2.4 [2.3, 2.6] & 4.8 [4.5, 5.0] \\
Qwen3-8B & Qwen3-0.6B & 3.244 & 12.3 [11.6, 13.1] & 14.2 [13.5, 14.9] & 4.822 & 4.5 [4.2, 4.7] & 5.4 [5.2, 5.7] \\
\bottomrule
\end{tabular*}
\caption{Standalone drafters in pools of EAGLE-3's shape. Draft is draft selection's expected acceptance. Root is root selection's gain over draft selection, and Root draw root-draw selection's gain over the PEARL-style decoder, with $95\%$ prompt-bootstrap intervals in brackets.}
\label{tab:lmdraft}
\end{table*}

\begin{table*}[t]
\centering\footnotesize
\setlength{\tabcolsep}{3pt}
\begin{tabular*}{\textwidth}{@{\extracolsep{\fill}}lrrrrr@{}}
\toprule
& & \multicolumn{4}{c}{Gain over draft selection} \\
\cmidrule(lr){3-6}
Target & $B$ & \cellcolor{tableroot}Exact & Cache & Same token & Preceding \\
\midrule
\rowcolor{tablegroup}\multicolumn{6}{c}{Greedy decoding} \\
Llama-3.1-8B & 8 & \cellcolor{tableroot}0.275 [0.259, 0.291] & $-$0.595 [$-$0.632, $-$0.559] & $-$0.220 [$-$0.243, $-$0.198] & $-$1.760 [$-$1.819, $-$1.695] \\
 & 59 & \cellcolor{tableroot}0.117 [0.110, 0.125] & $-$0.325 [$-$0.348, $-$0.302] & $-$0.111 [$-$0.124, $-$0.099] & $-$0.803 [$-$0.838, $-$0.770] \\
Qwen3-4B & 8 & \cellcolor{tableroot}0.315 [0.306, 0.323] & $-$0.289 [$-$0.304, $-$0.275] & $-$0.040 [$-$0.047, $-$0.033] & $-$0.936 [$-$0.954, $-$0.919] \\
 & 59 & \cellcolor{tableroot}0.166 [0.161, 0.172] & $-$0.185 [$-$0.195, $-$0.176] & $-$0.022 [$-$0.026, $-$0.017] & $-$0.519 [$-$0.532, $-$0.506] \\
Qwen3-8B & 8 & \cellcolor{tableroot}0.334 [0.325, 0.342] & $-$0.213 [$-$0.227, $-$0.199] & $-$0.015 [$-$0.023, $-$0.008] & $-$0.946 [$-$0.968, $-$0.922] \\
 & 59 & \cellcolor{tableroot}0.167 [0.162, 0.172] & $-$0.139 [$-$0.148, $-$0.130] & $-$0.007 [$-$0.012, $-$0.003] & $-$0.651 [$-$0.669, $-$0.634] \\
Qwen2.5-7B & 8 & \cellcolor{tableroot}0.285 [0.273, 0.296] & $-$0.442 [$-$0.466, $-$0.419] & $-$0.129 [$-$0.144, $-$0.116] & $-$1.240 [$-$1.278, $-$1.201] \\
 & 59 & \cellcolor{tableroot}0.124 [0.119, 0.130] & $-$0.268 [$-$0.283, $-$0.253] & $-$0.069 [$-$0.077, $-$0.061] & $-$0.599 [$-$0.623, $-$0.576] \\
Qwen2.5-14B & 8 & \cellcolor{tableroot}0.280 [0.268, 0.292] & $-$0.470 [$-$0.494, $-$0.447] & $-$0.143 [$-$0.158, $-$0.129] & $-$1.324 [$-$1.357, $-$1.287] \\
 & 59 & \cellcolor{tableroot}0.122 [0.116, 0.128] & $-$0.306 [$-$0.324, $-$0.288] & $-$0.078 [$-$0.087, $-$0.068] & $-$0.716 [$-$0.741, $-$0.691] \\
Qwen2.5-32B & 8 & \cellcolor{tableroot}0.273 [0.260, 0.285] & $-$0.479 [$-$0.504, $-$0.454] & $-$0.148 [$-$0.162, $-$0.133] & $-$1.387 [$-$1.427, $-$1.346] \\
 & 59 & \cellcolor{tableroot}0.119 [0.113, 0.125] & $-$0.291 [$-$0.309, $-$0.274] & $-$0.073 [$-$0.082, $-$0.064] & $-$0.712 [$-$0.738, $-$0.685] \\
\rowcolor{tablegroup}\multicolumn{6}{c}{Sampling at temperature $1$} \\
Llama-3.1-8B & 8 & \cellcolor{tableroot}0.210 [0.199, 0.221] & $-$0.411 [$-$0.440, $-$0.381] & $-$0.143 [$-$0.160, $-$0.127] & $-$1.322 [$-$1.392, $-$1.255] \\
 & 59 & \cellcolor{tableroot}0.107 [0.101, 0.112] & $-$0.225 [$-$0.241, $-$0.208] & $-$0.072 [$-$0.081, $-$0.064] & $-$0.646 [$-$0.679, $-$0.614] \\
Qwen3-4B & 8 & \cellcolor{tableroot}0.295 [0.287, 0.302] & $-$0.279 [$-$0.294, $-$0.265] & $-$0.041 [$-$0.047, $-$0.034] & $-$0.899 [$-$0.917, $-$0.881] \\
 & 59 & \cellcolor{tableroot}0.158 [0.154, 0.163] & $-$0.176 [$-$0.184, $-$0.167] & $-$0.021 [$-$0.026, $-$0.017] & $-$0.507 [$-$0.521, $-$0.494] \\
Qwen3-8B & 8 & \cellcolor{tableroot}0.313 [0.306, 0.320] & $-$0.199 [$-$0.212, $-$0.186] & $-$0.017 [$-$0.024, $-$0.011] & $-$0.872 [$-$0.893, $-$0.851] \\
 & 59 & \cellcolor{tableroot}0.160 [0.155, 0.164] & $-$0.130 [$-$0.140, $-$0.121] & $-$0.006 [$-$0.010, $-$0.002] & $-$0.609 [$-$0.624, $-$0.593] \\
Qwen2.5-7B & 8 & \cellcolor{tableroot}0.239 [0.230, 0.248] & $-$0.345 [$-$0.365, $-$0.325] & $-$0.096 [$-$0.107, $-$0.085] & $-$1.011 [$-$1.048, $-$0.975] \\
 & 59 & \cellcolor{tableroot}0.120 [0.115, 0.124] & $-$0.216 [$-$0.229, $-$0.203] & $-$0.052 [$-$0.058, $-$0.046] & $-$0.509 [$-$0.529, $-$0.488] \\
Qwen2.5-14B & 8 & \cellcolor{tableroot}0.226 [0.218, 0.234] & $-$0.373 [$-$0.394, $-$0.352] & $-$0.111 [$-$0.124, $-$0.099] & $-$1.087 [$-$1.122, $-$1.052] \\
 & 59 & \cellcolor{tableroot}0.115 [0.111, 0.120] & $-$0.257 [$-$0.271, $-$0.243] & $-$0.065 [$-$0.072, $-$0.058] & $-$0.611 [$-$0.632, $-$0.592] \\
Qwen2.5-32B & 8 & \cellcolor{tableroot}0.217 [0.209, 0.224] & $-$0.373 [$-$0.393, $-$0.352] & $-$0.107 [$-$0.119, $-$0.095] & $-$1.132 [$-$1.165, $-$1.097] \\
 & 59 & \cellcolor{tableroot}0.111 [0.108, 0.115] & $-$0.235 [$-$0.247, $-$0.223] & $-$0.051 [$-$0.057, $-$0.046] & $-$0.605 [$-$0.625, $-$0.584] \\
\bottomrule
\end{tabular*}
\caption{Gains over draft selection in expected accepted draft tokens, with $95\%$ intervals, when the softmax probabilities of the exact root row or of a row at another address replace the first draft factor. Rounds without one keep the draft ranking.}
\label{tab:approx-ci}
\end{table*}

\begin{table*}[t]
\centering\small
\begin{tabular*}{\textwidth}{@{\extracolsep{\fill}}lrrrrrrrr@{}}
\toprule
& \multicolumn{2}{c}{Available, \%} & \multicolumn{3}{c}{Mean TV} & \multicolumn{3}{c}{Argmax agree, \%} \\
\cmidrule(lr){2-3}\cmidrule(lr){4-6}\cmidrule(lr){7-9}
Target & Cache & Same token & Cache & Same token & Preceding & Cache & Same token & Preceding \\
\midrule
\rowcolor{tablegroup}\multicolumn{9}{c}{Greedy decoding} \\
Llama-3.1-8B & 76 & 40 & 0.46 & 0.36 & 0.94 & 54 & 64 & 6 \\
Qwen3-4B & 71 & 27 & 0.51 & 0.36 & 0.95 & 49 & 64 & 5 \\
Qwen3-8B & 71 & 29 & 0.52 & 0.35 & 0.94 & 49 & 65 & 6 \\
Qwen2.5-7B & 74 & 31 & 0.50 & 0.43 & 0.94 & 51 & 57 & 5 \\
Qwen2.5-14B & 72 & 30 & 0.49 & 0.39 & 0.94 & 51 & 60 & 5 \\
Qwen2.5-32B & 72 & 30 & 0.49 & 0.40 & 0.94 & 51 & 59 & 5 \\
\rowcolor{tablegroup}\multicolumn{9}{c}{Sampling at temperature $1$} \\
Llama-3.1-8B & 71 & 35 & 0.47 & 0.36 & 0.92 & 53 & 64 & 7 \\
Qwen3-4B & 70 & 26 & 0.52 & 0.36 & 0.95 & 49 & 64 & 5 \\
Qwen3-8B & 69 & 28 & 0.52 & 0.36 & 0.94 & 48 & 64 & 6 \\
Qwen2.5-7B & 69 & 28 & 0.52 & 0.43 & 0.93 & 48 & 57 & 6 \\
Qwen2.5-14B & 68 & 28 & 0.50 & 0.40 & 0.94 & 49 & 59 & 6 \\
Qwen2.5-32B & 68 & 27 & 0.50 & 0.39 & 0.94 & 49 & 60 & 6 \\
\bottomrule
\end{tabular*}
\caption{Agreement of rows at other addresses with the next round's root row. Availability is the share of roots with such a row, and agreement is computed over those roots. The preceding row is always available. Total variation uses softmax probabilities.}
\label{tab:approx-rows}
\end{table*}

\begin{figure*}[t]
\centering
\includegraphics[width=\textwidth]{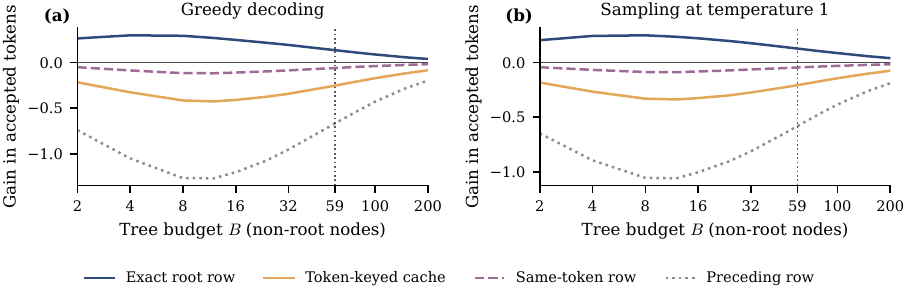}
\caption{Gain over draft selection, averaged over six targets, when the softmax probabilities of the exact root row or of a row at another address replace the first draft factor. The dotted line marks EAGLE-3's tree.}
\label{fig:approx}
\end{figure*}

\subsection{Calibrated and reranked draft scores}
\label{app:calib}
Table~\ref{tab:calib} compares two calibrated controls and two rerankers with draft and root selection. The first control replaces $q_1$ with an isotonic estimate of the target probability; the second calibrates every draft factor with a separate map for each depth. Each map is fitted over $200$ equal-count bins of draft probability on the even-numbered prompts and applied to the odd-numbered ones, and vice versa.

The reranker uses gradient-boosted trees fitted on one half of the prompts and applied to the other to predict the target probability of each first-level candidate, and under greedy decoding whether it is the target's token. Its inputs, all known before verification, are the draft probabilities of the candidate and its siblings, the draft distribution and pool subtree below the candidate, a token prior estimated on the training half, the task, the prompt length, the round's position, the previous round's acceptance and whether the candidate repeats the root's token.

The token prior averages the target over candidates with the same token and shrinks the average toward the overall mean with a pseudo-count of $20$. Within the training half, the candidates are split in alternation, and each one's prior comes from the other part, so no prior contains its own label. The model is a histogram gradient-boosted classifier under greedy decoding and a regressor under sampling, with $400$ iterations, learning rate $0.05$ and at most $63$ leaves. The prediction replaces the first draft factor as the root row does.

Its top-ranked first-level candidate is the target's most probable one in $60.1\%$ to $82.8\%$ of rounds, against $58.7\%$ to $82.7\%$ for the draft ranking, where greedy runs count only rounds whose first level contains the target's token.

The row features come from the token-keyed cache, a same-token row from the previous tree and the row that generated the continuation. For each available row, the reranker reads the candidate's log probability and its difference from the highest first-level log probability. It also reads the number of rounds since the cache row was stored and the number of previous-tree prefixes that end in the root's token. Each of these inputs is missing in rounds without its row, and the preceding row's entry for the continuation token is always missing.

This reranker recovers at most $29\%$ of the root gain at $B=8$ and $33\%$ at $B=59$, and in every run its held-out squared error on the first-level targets is below that of the draft probabilities and of the reranker without rows. In every run at both budgets, the $95\%$ intervals of its gain over draft selection and of its gain over the reranker without rows lie above zero. Their lower ends are at least $0.27\%$ and $0.30\%$ of draft selection's value at $B=59$, and $0.81\%$ and $0.88\%$ at $B=8$.

\subsection{Coverage from the scored frontier}
\label{app:frontier}
Scoring the $M$ highest-ranked candidates includes candidates whose parents lie outside the tree. Such a candidate cannot be the emitted address after exhaustive traversal. Table~\ref{tab:frontier} instead scores only frontier candidates, which lie outside the tree but have their parents in it, in draft order. Forty frontier candidates cover $12\%$ to $17\%$ of next roots, more than the $41$ candidates added by the top-$100$ set, and the complete frontier, $380$ to $418$ candidates on average, covers $28\%$ to $35\%$, the same share as the whole pool.

\subsection{Selection with reused rows}
\label{app:reuse}
Without a separate root pass, a decoder selects with the exact root row only after a covered round and keeps the draft ranking otherwise, including in the first round of each generation. Each round verifies the tree of the rule it uses and scores the $M$ highest-ranked candidates in that rule's order, so its coverage is that of its own tree. Under greedy decoding a round is covered when the greedy token after its accepted path labels a scored candidate. At temperature $1$ a round holds the row with the probability that the previous round was covered, and that round's coverage is the scored frontier mass of Equation~\ref{eq:frontier}. For EAGLE-3's tree this computation gives the recorded coverage in every greedy round.

At the recorded contexts, root selection in rounds that hold the row raises expected acceptance over draft ranking at $B=59$ by $0.5\%$ to $0.9\%$ when verification scores the $100$ highest-ranked candidates and by $1.1\%$ to $2.3\%$ when it scores the whole pool, and at $B=8$ by $4.9\%$ to $10.1\%$ and $5.2\%$ to $11.1\%$, with every interval above zero.

With the $100$ candidates, $41\%$ to $56\%$ of rounds hold the row at $B=8$, against $10\%$ to $15\%$ at $B=59$, since the smaller trees exit closer to the root and leave $92$ of the $100$ scored candidates outside the tree. Scoring the whole pool, $710$ prefixes, gives the largest coverage that any scored subset of the pool reaches. By Theorem~\ref{thm:run}, whatever the gain in rounds that hold the row, a decoder that is never covered and scores no more prefixes in each round has emitted at least as many tokens as such a decoder after every round.

\begin{table*}[t]
\centering\small
\begin{tabular*}{\textwidth}{@{\extracolsep{\fill}}lrrrrrr@{}}
\toprule
& & \multicolumn{2}{c}{Tokens for each pass} & & & \\
\cmidrule(lr){3-4}
Target & $B$ & EAGLE-3 & Covering & Change, \% & Covered, \% & Gain \\
\midrule
\rowcolor{tablegroup}\multicolumn{7}{c}{Greedy decoding} \\
Llama-3.1-8B & 8 & 4.171 & 3.859 & $-$7.5 [$-$8.2, $-$6.8] & 49 & 0.62 \\
 & 59 & 5.505 & 5.059 & $-$8.1 [$-$8.6, $-$7.6] & 52 & 0.25 \\
Qwen3-4B & 8 & 2.654 & 2.476 & $-$6.7 [$-$7.1, $-$6.3] & 31 & 0.49 \\
 & 59 & 3.323 & 3.106 & $-$6.5 [$-$6.9, $-$6.2] & 29 & 0.28 \\
Qwen3-8B & 8 & 2.569 & 2.395 & $-$6.8 [$-$7.2, $-$6.4] & 33 & 0.57 \\
 & 59 & 3.231 & 2.998 & $-$7.2 [$-$7.6, $-$6.8] & 31 & 0.34 \\
Qwen2.5-7B & 8 & 3.368 & 3.151 & $-$6.4 [$-$7.1, $-$5.8] & 35 & 0.61 \\
 & 59 & 4.302 & 3.987 & $-$7.3 [$-$8.1, $-$6.6] & 36 & 0.32 \\
Qwen2.5-14B & 8 & 3.446 & 3.229 & $-$6.3 [$-$6.8, $-$5.8] & 36 & 0.61 \\
 & 59 & 4.430 & 4.138 & $-$6.6 [$-$7.1, $-$6.0] & 38 & 0.29 \\
Qwen2.5-32B & 8 & 3.470 & 3.248 & $-$6.4 [$-$6.8, $-$6.0] & 36 & 0.59 \\
 & 59 & 4.446 & 4.152 & $-$6.6 [$-$7.1, $-$6.0] & 37 & 0.30 \\
\rowcolor{tablegroup}\multicolumn{7}{c}{Sampling at temperature $1$} \\
Llama-3.1-8B & 8 & 3.476 & 3.198 & $-$8.0 [$-$10.2, $-$5.8] & 42 & 0.67 \\
 & 59 & 4.493 & 4.124 & $-$8.2 [$-$11.2, $-$5.1] & 45 & 0.32 \\
Qwen3-4B & 8 & 2.594 & 2.419 & $-$6.8 [$-$7.3, $-$6.3] & 31 & 0.50 \\
 & 59 & 3.241 & 3.018 & $-$6.9 [$-$7.5, $-$6.3] & 29 & 0.30 \\
Qwen3-8B & 8 & 2.477 & 2.298 & $-$7.2 [$-$7.9, $-$6.5] & 32 & 0.59 \\
 & 59 & 3.104 & 2.869 & $-$7.5 [$-$8.3, $-$6.8] & 30 & 0.37 \\
Qwen2.5-7B & 8 & 2.962 & 2.803 & $-$5.4 [$-$6.6, $-$4.0] & 30 & 0.64 \\
 & 59 & 3.754 & 3.551 & $-$5.4 [$-$6.8, $-$3.8] & 32 & 0.36 \\
Qwen2.5-14B & 8 & 3.038 & 2.835 & $-$6.7 [$-$7.7, $-$5.6] & 32 & 0.65 \\
 & 59 & 3.866 & 3.630 & $-$6.1 [$-$7.4, $-$4.9] & 34 & 0.35 \\
Qwen2.5-32B & 8 & 3.031 & 2.866 & $-$5.5 [$-$6.6, $-$4.3] & 31 & 0.64 \\
 & 59 & 3.898 & 3.651 & $-$6.3 [$-$7.8, $-$4.8] & 34 & 0.36 \\
\bottomrule
\end{tabular*}
\caption{Covering decoder against EAGLE-3's decoder, in tokens emitted for each target pass. Change is the relative change with $95\%$ intervals, Covered the share of covered rounds, and Gain root-draw selection's token gain over draft selection after covered rounds.}
\label{tab:cover}
\end{table*}

We also run a covering decoder online, at its own contexts. It drafts EAGLE-3's pool and verifies its tree by exhaustive target sampling, except that a draw labeling a leaf of the verified tree, including a tree that holds only the drawn token, is emitted as the continuation instead of being accepted. The leaf's row from the same pass is then the next root row, from which the next round draws its first token before it selects its tree by root-draw selection, and other rounds keep EAGLE-3's tree. EAGLE-3's decoder runs on the same prompts with the same output limit.

Both run one target pass in each round, so Table~\ref{tab:cover} compares the tokens they emit for each target pass, counted up to the first end-of-sequence token. In the rounds that hold a covered row, the pass also scores the trees that draft selection and root selection would keep, for measurement only. The covering decoder emits fewer tokens for each target pass at both budgets, although selection improves in the rounds that hold the root row.

\subsection{Token and address matches}
\label{app:matches}
Among the retained rows of Appendix~\ref{app:setup}, the row at the emitted address is a matched row and the others are token-only rows. Table~\ref{tab:rowvalidity} gives the agreement of matched and token-only rows with the next round's root row in each run, and Table~\ref{tab:addr} the token and address match rates behind Section~\ref{sec:rows}. Matched pairs have median total variation below $0.0004$ in every run. Among token-only rows, those on the accepted path have mean total variation of $0.72$ to $0.79$ from the next root row, above that of the other token-only rows in every run, as they condition on an earlier instance of the continuation token.

\subsection{Rows at other addresses}
\label{app:approx}
Table~\ref{tab:approx-ci} reports direct first-factor substitution at both budgets, and Figure~\ref{fig:approx} varies the budget. All three substitutes lower acceptance relative to draft selection in these tests. Under greedy decoding, this comparison uses softmax probabilities even for the Exact column, whereas root selection uses the point-mass target law. Table~\ref{tab:approx-rows} reports availability and distribution agreement. The replay and row construction are described in Appendix~\ref{app:setup}.

The cache and preceding rows can also be mixed geometrically with the first draft factor, choosing the weight from $0.25$, $0.5$ and $0.75$ on the other half of the prompts. The preceding-row mixture still lowers acceptance, while the cache mixture yields no consistently positive gain.

\subsection{Standalone drafters}
\label{app:lm}
Llama-3.1-8B and Qwen3-8B also decode all $644$ prompts at both temperatures with a standalone drafter that shares the target's tokenizer, Llama-3.2-1B-Instruct \citep{meta2024llama32} for Llama-3.1-8B and Qwen3-0.6B \citep{qwen3technical2025} for Qwen3-8B. The drafter runs in bfloat16 with scaled dot-product attention instead of eager attention. It constructs the candidate pool in the same way as the EAGLE-3 drafter. Target settings, verification, scoring, prompts, output limit and seeding follow the protocol in Appendix~\ref{app:setup}. The selection rules of Section~\ref{sec:setup} therefore apply unchanged, and Table~\ref{tab:lmdraft} gives their values.

The relative root gain grows as the budget shrinks in every run, and the interval of the root gain lies above zero at each budget from $2$ to $200$. At $B=59$, the root row recovers $35\%$ to $57\%$ of the oracle's advantage, and the PEARL-style decoder exceeds the variant whose drafter commits to its most probable first token. Table~\ref{tab:lmdraft} therefore compares root-draw selection with the stronger baseline.

Against the PEARL-style decoder enlarged as in Appendix~\ref{app:selci} until the mean tree sizes are equal in each run, root-draw selection still gains $0.5\%$ to $1.3\%$ at $B=59$ and $4.2\%$ to $5.2\%$ at $B=8$, with every interval above zero.

\begin{table*}[t]
\centering\small
\begin{tabular*}{\textwidth}{@{\extracolsep{\fill}}llrrr@{}}
\toprule
& & \multicolumn{2}{c}{Over draft selection} & Over PEARL-style \\
\cmidrule(lr){3-4}\cmidrule(lr){5-5}
Target & Drafter & Root draw, \% & PEARL-style, \% & Root draw, \% \\
\midrule
\rowcolor{tablegroup}\multicolumn{5}{c}{Greedy decoding} \\
Llama-3.1-8B & Llama-3.2-1B & 12.0 [11.3, 12.7] & 2.1 [1.6, 2.7] & 9.6 [9.0, 10.3] \\
Qwen3-8B & Qwen3-0.6B & 17.4 [16.6, 18.3] & 4.9 [4.3, 5.5] & 12.0 [11.3, 12.6] \\
\rowcolor{tablegroup}\multicolumn{5}{c}{Sampling at temperature 1} \\
Llama-3.1-8B & Llama-3.2-1B & 19.3 [17.6, 21.1] & 7.3 [5.9, 8.8] & 11.2 [10.0, 12.3] \\
Qwen3-8B & Qwen3-0.6B & 19.1 [18.3, 20.0] & 6.7 [5.9, 7.4] & 11.7 [11.0, 12.4] \\
\bottomrule
\end{tabular*}
\caption{Gain in decoding speed at $B=8$ with each standalone drafter in its own process on a second GPU, over draft selection and over the PEARL-style decoder, with $95\%$ prompt-bootstrap intervals in brackets.}
\label{tab:pipe}
\end{table*}

\begin{table*}[t]
\centering\small
\begin{tabular*}{\textwidth}{@{\extracolsep{\fill}}lrrrrrrr@{}}
\toprule
& & \multicolumn{2}{c}{Nodes} & \multicolumn{2}{c}{Tree pass, ms} & & \\
\cmidrule(lr){3-4}\cmidrule(lr){5-6}
Target & Cap & Root draw & PEARL-style & Root draw & PEARL-style & Expected, \% & Speed, \% \\
\midrule
\rowcolor{tablegroup}\multicolumn{8}{c}{Greedy decoding} \\
Llama-3.1-8B & 43 & 42.0 & 41.6 & 32.9 & 33.1 & 1.1 [1.0, 1.3] & 1.3 [0.9, 1.7] \\
Qwen3-8B & 41 & 38.4 & 37.7 & 34.9 & 34.5 & 1.9 [1.7, 2.2] & 1.7 [1.2, 2.2] \\
\rowcolor{tablegroup}\multicolumn{8}{c}{Sampling at temperature 1} \\
Llama-3.1-8B & 40 & 36.1 & 35.7 & 31.8 & 32.1 & 2.1 [1.9, 2.2] & 2.2 [1.5, 3.0] \\
Qwen3-8B & 40 & 36.9 & 36.0 & 34.5 & 34.2 & 2.5 [2.3, 2.7] & 2.1 [1.6, 2.5] \\
\bottomrule
\end{tabular*}
\caption{Root-draw selection at its lowered cap against the PEARL-style decoder at $B=59$. Nodes and Tree pass are means over rounds. Expected is its gain in expected acceptance on the recorded rounds, Speed its gain in decoding speed, with $95\%$ intervals.}
\label{tab:pipe-matched}
\end{table*}

\subsection{Pipelined decoding}
\label{app:pipe}
Llama-3.1-8B and Qwen3-8B decode all $644$ prompts again with the standalone drafters of Appendix~\ref{app:lm}, now with the drafter in its own process on a second RTX A6000 GPU, as in PEARL's deployment \citep{liu2024pearl}. Target and drafter run at batch size one with a static key-value cache, and every call after the prompt prefill replays a captured CUDA graph. In each round, the target process sends the newly committed tokens to the drafter, which grows its pool and returns it. The drafter waits while the target verifies, so no decoder uses PEARL's post-verification, which drafts during verification.

Draft selection waits for the pool and verifies the continuation token and its $B$-node tree in one target pass. The PEARL-style decoder, root-draw selection and root selection run a one-token target pass on the continuation token while the drafter works. They draw the first token from the returned root row and then verify their trees as described in Appendix~\ref{app:online}. Root selection ranks the pool with this row and verifies its whole tree. When the drawn token is not on the tree's first level, which happens in $1.2\%$ to $8.4\%$ of its rounds at $B=59$, the drawn token becomes the round's continuation and the second pass is skipped.

The settings are $T\in\{0,1\}$ and $B\in\{8,59\}$, and a decode stops at its first end-of-sequence token or after the round in which its output reaches the default limit. After warm-up decodes on two prompts, each prompt runs all sixteen combinations of decoder, $T$ and $B$ once, in an order rotated by its position in a seeded prompt order, so resampling prompts also resamples the timing noise of single decodes.

A decode's speed counts the tokens emitted up to the first end-of-sequence token over the time from the start of the target prefill to the last commit, so it includes the prefill, and a run's speed divides the summed tokens by the summed time. Gains compare the speeds of two decoders on the same prompts, with $95\%$ percentile intervals from $2{,}000$ paired bootstrap resamples of prompts. All prompts are retained in the timing comparison.

Drafting takes a median of $50$ to $51$ ms a round and the one-token target pass $26$ to $28$ ms, so that pass ends while the drafter works, and the median round of root-draw or root selection lies within $0.4$ ms of draft selection's. At $B=59$, root selection decodes $3.2\%$ to $6.2\%$ faster than draft selection, root-draw selection $3.5\%$ to $9.6\%$ faster and the PEARL-style decoder $4.3\%$ to $8.3\%$ faster, and no lower end of their intervals falls below $2.7\%$. At $B=8$, root selection decodes $9.6\%$ to $17.3\%$ faster than draft selection, and Table~\ref{tab:pipe} gives the gains of root-draw selection and the PEARL-style decoder, whose intervals also exclude zero.

In these runs, root-draw selection's trees hold $16$ to $18$ more nodes on average than the PEARL-style decoder's at $B=59$ and $1.4$ to $2.0$ more at $B=8$. Enlarging the PEARL-style tree at $B=59$ to equal mean size would leave more than $64$ tokens to verify after pruning, beyond the largest captured CUDA graph, in $45\%$ to $51\%$ of rounds. To time the two decoders at matched tree sizes, we lower root-draw selection's cap instead, in each run to the smallest value at which its mean expected tree size on the recorded rounds of Appendix~\ref{app:lm} still reaches the PEARL-style decoder's at $B=59$, between $40$ and $43$ nodes.

Both decoders then decode all $644$ prompts again in the same deployment, interleaved as before, with two sampling seeds at $T=1$ and all prompts retained. In Table~\ref{tab:pipe-matched}, root-draw selection verifies at least as many nodes on average as the PEARL-style decoder and decodes $1.3\%$ to $2.2\%$ faster. Its expected acceptance at these caps on the recorded rounds is $1.1\%$ to $2.5\%$ higher. These offline gains exceed the equal-size gains of Appendix~\ref{app:lm}, which match the two decoders at root-draw selection's larger trees.

\end{document}